\documentclass{article}

\usepackage{microtype}
\usepackage{graphicx}
\usepackage{booktabs} 
\usepackage{multirow}
\usepackage{array}
\usepackage{float} 
\usepackage{xurl}  
\newcommand{\Rad}{\mathcal{R}}

\usepackage{hyperref}

\usepackage[preprint]{icml2026}

\hypersetup{
  pdfsubject={},
  pdfkeywords={Machine Learning, Symbolic Regression}
}
\usepackage{booktabs}
\usepackage{siunitx}
\usepackage[table]{xcolor}
\usepackage{amsmath}
\usepackage{amssymb}
\usepackage{mathtools}
\usepackage{amsthm}
\usepackage{booktabs} 
\usepackage{tabularx}  
\usepackage{makecell}  
\usepackage[most]{tcolorbox}
\usepackage{listings}
\usepackage{enumitem}

\newtcolorbox{promptbox}[2][]{%
    breakable,  
    colback=gray!5!white,      
    colframe=gray!75!black,    
    fonttitle=\bfseries,       
    title={#2},                
    enhanced,                  
    attach boxed title to top left={yshift=-2mm, xshift=2mm},
    boxed title style={sharp corners, colback=gray!75!black},
    before upper=\raggedright,
    fontupper=\small\ttfamily, 
    #1
}

\usepackage[capitalize,noabbrev]{cleveref}

\theoremstyle{plain}
\newtheorem{theorem}{Theorem}[section]
\newtheorem{proposition}[theorem]{Proposition}
\newtheorem{lemma}[theorem]{Lemma}
\newtheorem{corollary}[theorem]{Corollary}
\theoremstyle{definition}
\newtheorem{definition}[theorem]{Definition}
\newtheorem{assumption}[theorem]{Assumption}
\theoremstyle{remark}

\usepackage[textsize=tiny]{todonotes}

\icmltitlerunning{Prior-Guided Symbolic Regression}

\begin{document}

\twocolumn[
  \icmltitle{Prior-Guided Symbolic Regression: Towards Scientific Consistency in Equation Discovery}



  \icmlsetsymbol{corresponding}{*}

  \begin{icmlauthorlist}
    \icmlauthor{Jing Xiao}{sch}
    \icmlauthor{Xinhai Chen}{sch}
    \icmlauthor{Jiaming Peng}{sch}
    \icmlauthor{Qinglin Wang}{sch}
    \icmlauthor{Menghan Jia}{sch}
    \icmlauthor{Zhiquan Lai}{sch}
    \icmlauthor{Guangping Yu}{sch}
    \icmlauthor{Dongsheng Li}{sch}
    \icmlauthor{Tiejun Li}{sch}
    \icmlauthor{Jie Liu}{sch}
  \end{icmlauthorlist}

  \icmlaffiliation{sch}{College of Computer Science and Technology, National University of Defense Technology, Changsha 410073, China}

  \icmlcorrespondingauthor{Xinhai Chen}{chenxinhai16@nudt.edu.cn}

  \icmlkeywords{Machine Learning, Symbolic Regression}

  \vskip 0.3in
]



\NoHyper\printAffiliationsAndNotice{}\endNoHyper  
\begin{abstract}
Symbolic Regression (SR) aims to discover interpretable equations from observational data, with the potential to reveal underlying principles behind natural phenomena.
However, existing approaches often fall into the Pseudo-Equation Trap: producing equations that fit observations well but remain inconsistent with fundamental scientific principles. A key reason is that these approaches are dominated by empirical risk minimization, lacking explicit constraints to ensure scientific consistency.
To bridge this gap, we propose PG-SR, a prior-guided SR framework built upon a three-stage pipeline consisting of warm-up, evolution, and refinement. Throughout the pipeline, PG-SR introduces a prior constraint checker that explicitly encodes domain priors as executable constraint programs, and employs a Prior Annealing Constrained Evaluation (PACE) mechanism during the evolution stage to progressively steer discovery toward scientifically consistent regions.
Theoretically, we prove that PG-SR reduces the Rademacher complexity of the hypothesis space, yielding tighter generalization bounds and establishing a guarantee against pseudo-equations. Experimentally, PG-SR outperforms state-of-the-art baselines across diverse domains, maintaining robustness to varying prior quality, noisy data, and data scarcity.
\end{abstract}

\section{Introduction}
\begin{figure}[t]
    \centering
    \includegraphics[width=0.48\textwidth]{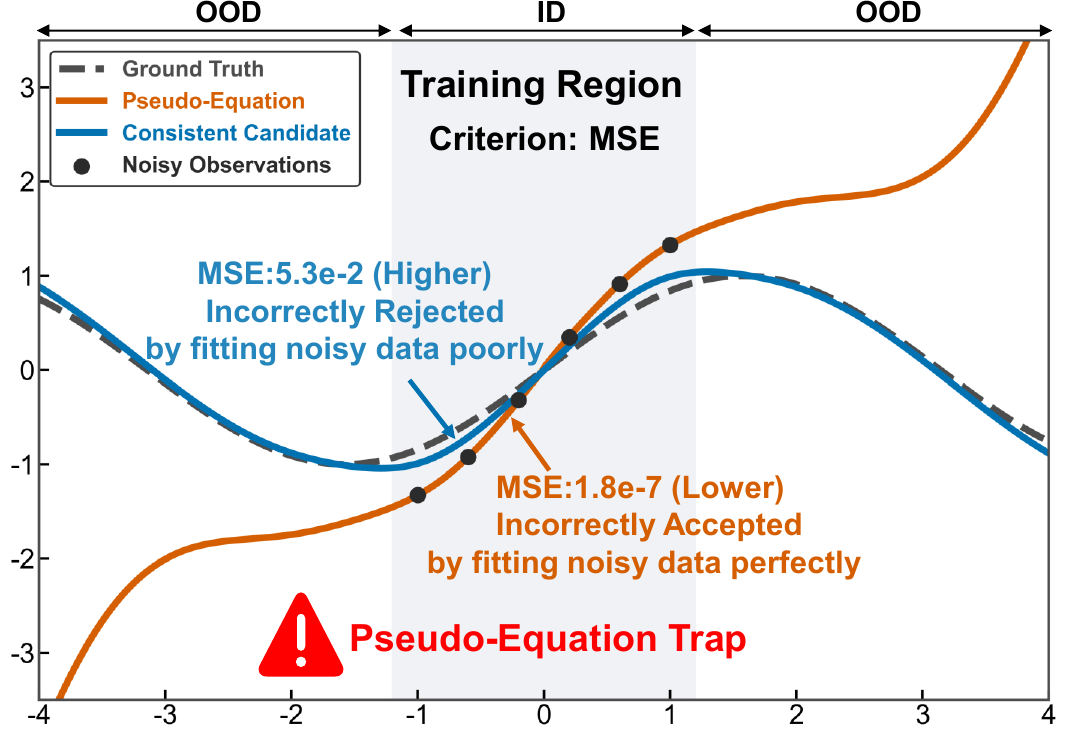}
\caption{The Pseudo-Equation Trap. The pseudo-equation (orange) is incorrectly accepted as it fits noisy training data perfectly, despite deviating from the true underlying equation. In contrast, the consistent candidate (blue) is incorrectly rejected because it fits noise poorly, even though it aligns with the true equation.}
    \label{FIG:1}
\end{figure} 

Symbolic Regression (SR) aims to extract interpretable mathematical equations from observational data~\cite{schmidt2009distilling}, revealing underlying principles behind natural phenomena and facilitating scientific discovery~\cite{wang2019symbolic,udrescu2020ai,makke2024interpretable}.
Unlike black-box models driven solely by predictive accuracy~\cite{de2024ai}, SR prioritizes interpretability and consistency with fundamental scientific principles, enabling reliable extrapolation beyond training distribution rather than merely fitting data~\cite{cranmer2020discovering}.

However, existing SR approaches often fall into the Pseudo-Equation Trap: producing equations that fit observations well but remain inconsistent with fundamental scientific principles. This problem is further worsened in real-world scenarios, where observations are often noisy and scarce, making pseudo-equations more likely to be produced and leading to catastrophic divergence under out-of-distribution (OOD) conditions~\cite{karniadakis2021physics}. As illustrated in Figure~\ref{FIG:1}, scientifically consistent candidates can be incorrectly rejected due to higher errors caused by noise, while structurally incorrect pseudo-equations are accepted.

Current approaches struggle to escape this trap because their optimization objectives are dominated by empirical risk minimization~\cite{arjovsky2019invariant}, lacking explicit constraints to ensure scientific consistency and thereby degenerating equation discovery into empirical fitting rather than principled scientific modeling.
Specifically, search-based approaches~\cite{mccormick2019gplearn,burlacu2020operon,virgolin2021improving,cranmer2023interpretable,petersen2019deep,landajuela2022unified}
impose only implicit constraints through operator design and complexity heuristics, and therefore remain dominated by empirical risk minimization during the search.
Transformer-based approaches~\cite{biggio2021neural,kamienny2022end,shojaee2023transformer,ying2025neural}
rely on pre-training on synthetic data to induce inductive biases toward scientific consistency, which function as  constraints; however, these constraints collapse under distributional shift, degenerating the discovery process into merely empirical fitting.
LLM-based approaches~\cite{grayeli2024symbolic,shojaee2024llm,wang2025drsr}
use LLMs as surrogates to impose scientific consistency constraints, but this guidance is prompt-based and thus remains implicit, making it vulnerable to hallucinations~\cite{ji2023survey}, leading to pseudo-equations. Therefore, when empirical risk minimization governs the discovery, hallucinated pseudo-equations can emerge in early stages and be reused as context in subsequent iterations, leading to their progressive amplification across generations.

To address these challenges, we propose PG-SR, a prior-guided SR framework built upon a three-stage pipeline consisting of warm-up, evolution, and refinement. Throughout the pipeline, PG-SR introduces a prior constraint checker that
explicitly encodes domain priors as executable
constraint programs, and employs a Prior-Annealing Constrained Evaluation (PACE) mechanism during the evolution stage to progressively steer discovery toward scientifically consistent regions. This approach mitigates the dominance of empirical risk minimization and safeguards against the Pseudo-Equation Trap.

Our main contributions are summarized as follows:
\begin{itemize}
\item \textbf{Problem Formulation:} We formalize the Pseudo-Equation Trap, which reveals the fundamental divergence that arises between empirical risk minimization and scientific consistency in equation discovery.
\item \textbf{Method:} We propose PG-SR, a prior-guided SR framework that explicitly incorporates constraints to ensure scientific consistency, mitigating the dominance of empirical risk minimization.
\item \textbf{Theory:} We show that PG-SR reduces the Rademacher complexity of the hypothesis space, yielding tighter generalization bounds and establishing a theoretical guarantee against pseudo-equations.
\item \textbf{Experiments:} Extensive experiments across different domains demonstrate that PG-SR outperforms existing approaches, maintaining robustness to varying prior quality, noisy data, and data scarcity.
\end{itemize}

\section{Related Work}
\subsection{Search-Based Symbolic Regression}
Search-based approaches formulate symbolic regression as a combinatorial optimization problem. Genetic programming approaches~\cite{koza1994genetic}, such as GPlearn~\cite{mccormick2019gplearn}, Operon~\cite{burlacu2020operon}, GP-GOMEA~\cite{virgolin2021improving}, and PySR~\cite{cranmer2023interpretable}, perform heuristic searches by evolving expression trees, while approaches like DSR~\cite{petersen2019deep} and uDSR~\cite{landajuela2022unified} employ reinforcement learning for sequential generation~\cite{uc2023survey}. Despite their effectiveness, these approaches impose only implicit constraints through operator design and complexity heuristics, allowing empirical risk minimization to dominate search process and leading to pseudo-equations.

\subsection{Transformer-Based Symbolic Regression}
Transformer-based approaches, including NeSymReS~\cite{biggio2021neural}, E2E~\cite{kamienny2022end}, and Symformer~\cite{vastl2024symformer}, utilize large-scale pretraining on synthetic datasets to enable end-to-end equation generation. TPSR~\cite{shojaee2023transformer} further integrates Monte Carlo Tree Search–based planning~\cite{kocsis2006bandit} into the decoding process. Building on this paradigm, PhyE2E~\cite{ying2025neural} incorporates second-order derivative decomposition and dimensional constraints to improve physical consistency. These pretraining strategies induce inductive biases toward scientific consistency from synthetic data, which function as constraints; however, these constraints are tied to the training distribution and collapse under distributional shift, allowing empirical risk minimization to dominate discovery and resulting in pseudo-equations.

\subsection{LLM-Based Symbolic Regression}
Recently, leveraging LLMs~\cite{brown2020language,achiam2023gpt,team2023gemini} for symbolic regression has emerged as a new direction. LASR~\cite{grayeli2024symbolic} employs an LLM-extracted semantic concept library to guide discovery; LLM-SR~\cite{shojaee2024llm} and FunSearch~\cite{romera2024mathematical} exploit the code generation abilities of LLMs~\cite{chen2021evaluating} to construct equations; DrSR \cite{wang2025drsr} proposes a dual-reasoning mechanism that combines data-driven insights with reflective feedback to guide LLM generation. Although these methods use LLMs as surrogates to impose scientific consistency constraints, such guidance is prompt-based and remains implicit, rendering it vulnerable to hallucinations~\cite{ji2023survey} and leading to pseudo-equations. When empirical risk minimization dominates discovery, pseudo-equations can emerge in early stages, be reused as context, and propagate across generations, amplifying these hallucinations.

\begin{figure*}[t]
	\centering
		\includegraphics[width=0.99\textwidth,height=8.6cm]{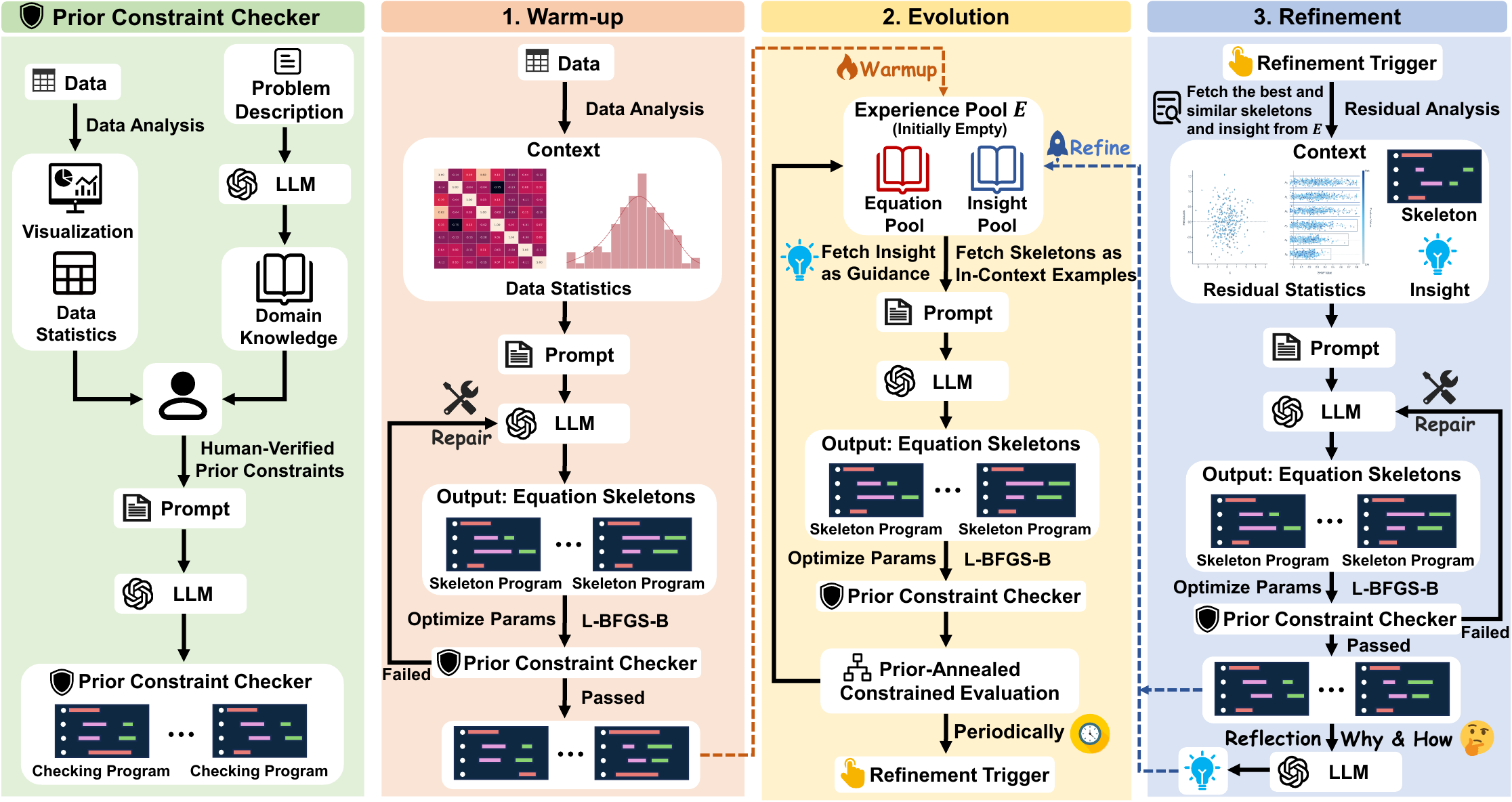}
\caption{Overview of the PG-SR framework with a three-stage pipeline comprising Warm-up, Evolution, and Refinement. The left panel illustrates the construction of the prior constraint checker.}
	\label{FIG:2}
\end{figure*} 
\section{Methodology}
\subsection{Problem Formulation}
Let $\mathcal{D} = \{ (\mathbf{x}_i, y_i) \}_{i=1}^N$ denote the dataset. The objective is to identify a function $f^* \in \mathcal{H}$, where $\mathcal{H}$ denotes the hypothesis space, that maximizes a score based on the Mean Squared Error (MSE), denoted by $S_{\text{MSE}}(f,\mathcal{D})$, defined as
\begingroup
\small
\begin{equation}
S_{\text{MSE}}(f,\mathcal{D})
= -\text{MSE}(f, \mathcal{D})
= - \frac{1}{N} \sum_{i=1}^N \big( y_i - f(\mathbf{x}_i) \big)^2,
\end{equation}
\endgroup
\begin{equation}
f^* = \arg\max_{f \in \mathcal{H}} S_{\text{MSE}}(f,\mathcal{D}).
\end{equation}
To incorporate prior knowledge, let $\mathcal{C}$ denote a set of scientific prior constraints.
We then define a boolean prior constraint checker $V$ as
\begin{equation}
V(f, \mathcal{C}) =
\begin{cases}
1, & \text{if $f$ satisfies all constraints in $\mathcal{C}$},\\
0, & \text{otherwise}.
\end{cases}
\end{equation}

To jointly account for numerical accuracy and prior constraints, we define a total score for each candidate function:
\begin{equation}
S(f, \mathcal{D}, \mathcal{C}) = S_{\text{MSE}}(f,\mathcal{D}) \cdot w(V(f, \mathcal{C})),
\label{eq4}
\end{equation}
where $w(V(f, \mathcal{C}))$ assigns greater emphasis to candidates satisfying the constraints while retaining contributions from those that violate some constraints to preserve diversity. The symbolic regression problem is thus reformulated as
\begin{equation}
f^* = \arg\max_{f \in \mathcal{H}} S(f, \mathcal{D}, \mathcal{C}).
\end{equation}

\subsection{Theoretical Analysis}
\label{sec:theory}
We formalize the PG-SR within the framework of statistical learning theory~\cite{mohri2018foundations}, define the Pseudo-Equation Trap, and demonstrate how incorporating explicit prior constraints can effectively mitigate it.

\begin{definition}[Symbolic Hypothesis Space]
Let $\mathcal{X} \subseteq \mathbb{R}^d$ and $\mathcal{Y} \subseteq \mathbb{R}$ be the input and output spaces. The hypothesis space $\mathcal{H}$ is defined as the set of all valid symbolic expression trees bounded by a maximum depth $L$.
\end{definition}

Standard symbolic regression seeks a function $f^* \in \mathcal{H}$ that minimizes the empirical risk $\hat{R}_N(f)$. However, as established in statistical learning theory~\cite{mohri2018foundations}, minimizing empirical risk does not guarantee minimizing the expected risk $R(f)$.

\begin{lemma}[Generalization Bound]
\label{lem:gen_bound}
Given the hypothesis space $\mathcal{H}$, assume the loss function is $\lambda$-Lipschitz and bounded by $M$. For any $\delta > 0$, with probability at least $1-\delta$, the expected risk $R(f)$ satisfies:
\begin{equation}
    R(f) \le \hat{R}_N(f) + 2\lambda\mathcal{R}_N(\mathcal{H}) + M\sqrt{\frac{\log(1/\delta)}{2N}}
\end{equation}
where $\mathcal{R}_N(\mathcal{H})$ denotes the Rademacher complexity of $\mathcal{H}$. (See Appendix~\ref{sec:proof_thm_3_2} for the detailed formulation.)
\end{lemma}

Based on this standard bound, we can formally derive the origin of pseudo-equations.

\begin{corollary}[The Pseudo-Equation Trap]
\label{coro:pseudo}
In symbolic regression, the hypothesis space $\mathcal{H}$ grows exponentially with depth $L$, causing $\mathcal{R}_N(\mathcal{H})$ to be excessively large. 
This looseness in the generalization bound allows for the existence of functions $f_{\text{pseudo}} \in \mathcal{H}$ such that $\hat{R}_N(f_{\text{pseudo}}) \approx 0$ (fitting data perfectly) while $R(f_{\text{pseudo}}) \gg 0$ (violating scientific principles). We  define such functions as \textit{pseudo-equations}.
\end{corollary}

To address this, PG-SR restricts the discovery process to a scientifically consistent subspace, as defined below.
\begin{definition}[Prior-Constrained Subspace]
Given a set of prior constraints $\mathcal{C}$ verified by domain experts, the constrained subspace is $\mathcal{H}_{\mathcal{C}} = \{f \in \mathcal{H} \mid V(f, \mathcal{C})=1\}$.
\end{definition}

\begin{proposition}[Consistency-Guaranteed Generalization]
\label{prop:complexity}
Assuming  $f_{\text{true}} \in \mathcal{H}_{\mathcal{C}}$, enforcing constraints $\mathcal{C}$  reduces the Rademacher complexity:
\begin{equation}
    \Rad_N(\mathcal{H}_{\mathcal{C}}) \le  \Rad_N(\mathcal{H}).
\end{equation}
Consequently, searching in $\mathcal{H}_{\mathcal{C}}$ yields a  tighter generalization bound than searching in $\mathcal{H}$. The formal proof is provided in Appendix~\ref{sec:proof_prop_3_5}.
\end{proposition}

Proposition~\ref{prop:complexity} reveals that the effectiveness of PG-SR depends not on reducing the size of the search space, but on aligning $\mathcal{H}_{\mathcal{C}}$ with the underlying truth. Unlike approaches that incorporate constraints only implicitly and may exclude the ground truth, PG-SR employs explicit, human-verified constraints to prune pseudo-equation regions of $\mathcal{H}$ while preserving the ground truth.

\subsection{Framework Overview}
PG-SR is a prior-guided SR framework with a three-stage pipeline: warm-up, evolution, and refinement (Figure~\ref{FIG:2}).
 The framework manages candidate equations in a global experience pool $E$, which includes an equation pool of  candidate equations maintained using logically partitioned islands~\cite{whitley1999island} to preserve diversity, and an insight pool for storing insights from success and failure experiences. In each iteration, LLMs generate new hypotheses based on high-scoring historical candidates and insights from previous iterations; the resulting candidates are subsequently optimized, evaluated, and used to update the pools. This framework ensures continuous evolution in both scientific consistency and numerical accuracy. Detailed prompt templates are provided in Appendix~\ref{app:prompts}.

\subsubsection{Prior-constraint Checker and Prior-Annealed Constrained Evaluation}
To ensure scientific consistency, we introduce a prior constraint checker \(V\) that validates candidate equations against a set of predefined constraints \(\mathcal{C}\). Since directly constructing these prior constraints is challenging, we adopt an LLM-assisted workflow to generate them (Figure~\ref{FIG:2}). First, an LLM extracts domain knowledge from the problem description. Then, human experts combine the LLM outputs with statistical patterns derived from data analysis to filter candidate priors. Finally, these validated priors are used by the LLM to generate the final usable prior constraints, which are represented in executable program form.
For example, in modeling \textit{E. coli} growth dynamics, the LLM extracts relevant biological knowledge from the problem description, such as the inhibition of growth under extreme temperature or pH conditions; human experts integrate this knowledge with statistical patterns to verify candidate priors, after which the LLM translates the validated priors into executable prior constraint programs. These constraints conform to established biological principles rather than relying on ground-truth equations. Further details on the  priors are provided in Appendix~\ref{app:priors}.

While explicit prior constraints can mitigate pseudo-equations, applying them too rigidly early in the discovery process can restrict exploration. Therefore, we propose the Prior-Annealed Constrained Evaluation (PACE) mechanism, which allows temporary violations initially and gradually tightens enforcement as evolution proceeds. 
PACE instantiates the prior weighting term $w(\cdot)$ in Eq.~(\ref{eq4}) via a shrink-and-shift
transformation of $S_{\text{MSE}}$.
Specifically, $S_{\text{MSE}}$ is min-max normalized within each island to make candidates more distinguishable:
$\tilde{S}_{\text{MSE}}(f,\mathcal{D}) = \frac{S_{\text{MSE}} - \min S_{\text{MSE}}}{\max S_{\text{MSE}} - \min S_{\text{MSE}}}$,
where $\min S_{\text{MSE}}$ and $\max S_{\text{MSE}}$ are the minimum and maximum scores in the island.
The total score with PACE is then defined as
\begingroup
\small
\begin{equation}
S(f,\mathcal{D},\mathcal{C}) =
\begin{cases}
\sigma(t)((1-\beta)+2\beta \tilde{S}_{\text{MSE}}) - \delta(t), & V=0,\\
(1-\beta)+2\beta \tilde{S}_{\text{MSE}}, & V=1,
\end{cases}
\end{equation}
\endgroup
where $\beta$ is a fixed scaling parameter.  
For constraint-violating candidates, the score is transformed as
\begingroup
\small
\begin{equation}
\phi(t) = \frac{B^t - 1}{B - 1}, \quad
\sigma(t) = [1 - \eta \cdot \phi(t)]_+, \quad
\delta(t) = \alpha \cdot \phi(t).
\end{equation}
\endgroup
with $t = N_{\mathrm{curr}} / N_{\mathrm{max}}$, where
$N_{\mathrm{curr}}$ and $N_{\mathrm{max}}$ are the current and total sampled candidates. Further details on the PACE mechanism can be found in the appendix~\ref{app:pace}.

\subsubsection{Stage I: Data-driven Warm-up}\label{sec:warmup}
The Warm-up stage initializes the experience pool $E$ with a data-driven set of preliminary equations, providing a solid foundation for subsequent evolution.
Statistical and structural analyses of the dataset $\mathcal{D}$, including sampled points, variable ranges, correlations, and nonlinear contributions, extract key data features to guide the LLM in generating equation skeletons aligned with the underlying system behavior.
Skeletons are then optimized with L-BFGS-B~\cite{byrd1995limited} to minimize MSE, and all candidates are retained regardless of prior constraints. 
If none of the generated skeletons satisfy the  prior constraints, a repair procedure is triggered to modify them. Finally, candidates are assigned cyclically to islands to ensure even distribution.

\begin{table*}[t]
\centering
\caption{Quantitative results of different approaches. ``-'' indicates that the method failed to generate equations.}
\label{tab:sr_baseline_comparison}
\small
\setlength{\tabcolsep}{2.3pt} 
\renewcommand{\arraystretch}{1.03}

\begin{tabular}{l cccccccccc}
\toprule
\textbf{Method} 
& \multicolumn{2}{c}{\textbf{E. coli Growth}} 
& \multicolumn{2}{c}{\textbf{Stress--Strain}} 
& \multicolumn{2}{c}{\textbf{CRK}} 
& \multicolumn{2}{c}{\textbf{Oscillator 1}} 
& \multicolumn{2}{c}{\textbf{Oscillator 2}} \\
\cmidrule(lr){2-3} \cmidrule(lr){4-5} \cmidrule(lr){6-7} \cmidrule(lr){8-9} \cmidrule(lr){10-11}
& ID $\downarrow$ & OOD $\downarrow$ & ID $\downarrow$ & OOD $\downarrow$ & ID $\downarrow$ & OOD $\downarrow$ & ID $\downarrow$ & OOD $\downarrow$ & ID $\downarrow$ & OOD $\downarrow$ \\
\midrule

\rowcolor[gray]{0.95} \multicolumn{11}{c}{\textit{Search-based approaches}} \\
GPLearn & 1.07e+00 & 1.03e+00 & 3.94e-01 & 1.00e+00 & 1.09e+00 & 1.01e+00 & 9.75e-03 & 5.51e-01 & 1.87e-01 & 2.71e-01 \\
PySR & 1.51e-01 & 7.10e-01 & 1.87e-02 & 7.72e-02 & 1.33e-09 & 2.21e-08 & 6.12e-12 & 2.55e-05 & 4.40e-10 & 2.05e-06 \\
DSR & 1.82e-01 & 3.28e-01 & 3.42e-01 & 7.94e-01 & 1.78e-01 & 3.80e+00 & 1.04e-02 & 3.80e-01 & 1.87e-01 & 2.82e-01 \\
uDSR & 5.10e-01 & 2.02e+00 & 8.81e-02 & 4.93e-01 & 2.54e-10 & 2.11e-07 & 1.41e-05 & 1.65e-02 & 1.89e-03 & 2.93e-02 \\
Operon & 4.52e-01 & 9.98e-01 & 3.46e-02 & 1.15e-01 & 1.75e-07 & 2.33e-07 & 3.78e-04 & 3.80e-03 & 1.74e-01 & 1.02e-01 \\
GP-GOMEA & 3.68e-01 & 4.51e+01 & 7.09e-02 & 2.15e-01 & 1.13e-08 & 1.36e-07 & 1.41e-03 & 1.29e+00 & 1.50e-01 & 1.63e-01 \\

\midrule
\rowcolor[gray]{0.95} \multicolumn{11}{c}{\textit{Transformer-based approaches}} \\
TPSR & -- & -- & 6.02e-01 & 1.95e+00 & 8.48e-07 & 2.44e-04 & 4.81e-03 & 5.13e-01 & 3.07e-01 & 9.01e-01 \\
E2E & 8.91e-01 & 8.32e-01 & 1.16e-01 & 5.12e-01 & 2.51e-03 & 3.21e-01 & 4.01e-03 & 4.77e-02 & 1.88e-01 & 7.18e-01 \\
NeSymReS & -- & -- & 8.58e-01 & 7.46e-01 & 9.70e-01 & 3.64e+02 & 4.01e-03 & 5.42e-01 & 9.85e-01 & 9.99e-01 \\
PhyE2E & 2.14e-01 & 5.33e-01 & 3.21e-02 & 1.47e-01 & 1.26e-05 & 1.79e-03 & 5.93e-03 & 1.65e-02 & 9.93e-02 & 6.36e-01 \\

\midrule
\rowcolor[gray]{0.95} \multicolumn{11}{c}{\textit{LLM-based approaches (Backbone: Llama-3.3-70B)}} \\
LLM-SR & 4.61e-03 & 2.19e-02 & 1.40e-02 & 6.77e-02 & 1.32e-08 & 4.55e-07 & 1.13e-05 & 1.11e-02 & 8.99e-08 & 1.60e-05 \\
LaSR & 1.81e-02 & 2.24e-02 & 1.64e-02 & 6.18e-02 & 9.55e-11 & 1.72e-07 & 7.15e-06 & 5.44e-03 & 6.63e-06 & 8.80e-04 \\
DrSR & 1.92e-02 & 2.77e-02 & 2.11e-02 & 7.30e-02 & 1.64e-10 & 8.30e-08 & 7.73e-13 & 1.87e-06 & 9.51e-05 & 1.92e-03 \\
\textbf{PG-SR(Ours)} & \colorbox{gray!30}{\textbf{1.15e-03}} & 3.37e-03 & \colorbox{gray!30}{\textbf{6.14e-03}} & \colorbox{gray!30}{\textbf{3.21e-02}} & 6.02e-11 & 3.21e-08 & 3.84e-12 & 5.64e-07 & 3.13e-09 & 7.17e-07 \\

\midrule
\rowcolor[gray]{0.95} \multicolumn{11}{c}{\textit{LLM-based approaches (Backbone: GPT-4o-mini)}} \\
LLM-SR & 9.34e-03 & 1.22e-02 & 2.10e-02 & 8.60e-02 & 1.14e-07 & 1.61e-04 & 4.24e-09 & 1.63e-05 &2.32e-11 & 5.08e-07 \\
LaSR & 8.46e-03 & 1.00e-01 & 2.04e-02 & 8.77e-02 & 3.01e-09 & 7.25e-08 & 1.13e-05 & 9.66e-03 & 2.68e-08 & 3.51e-05 \\
DrSR & 9.49e-02 & 2.23e-01 & 2.42e-02 & 1.41e-01 & 1.42e-10 & 3.44e-08 & 1.45e-12 & 3.78e-06 & 1.43e-04 & 1.04e-02 \\
\textbf{PG-SR(Ours)} & 1.32e-03 & \colorbox{gray!30}{\textbf{1.59e-03}} & 1.49e-02 & 4.63e-02 & \colorbox{gray!30}{\textbf{1.74e-11}} & \colorbox{gray!30}{\textbf{1.22e-08}} & \colorbox{gray!30}{\textbf{6.62e-14}} & \colorbox{gray!30}{\textbf{2.95e-08}} &  \colorbox{gray!30}{\textbf{1.79e-11}} & \colorbox{gray!30}{\textbf{8.32e-11}} \\
\bottomrule
\end{tabular}
\end{table*}

\begin{figure*}[t]
	\centering
		\includegraphics[width=0.9\textwidth,height=13.3cm]{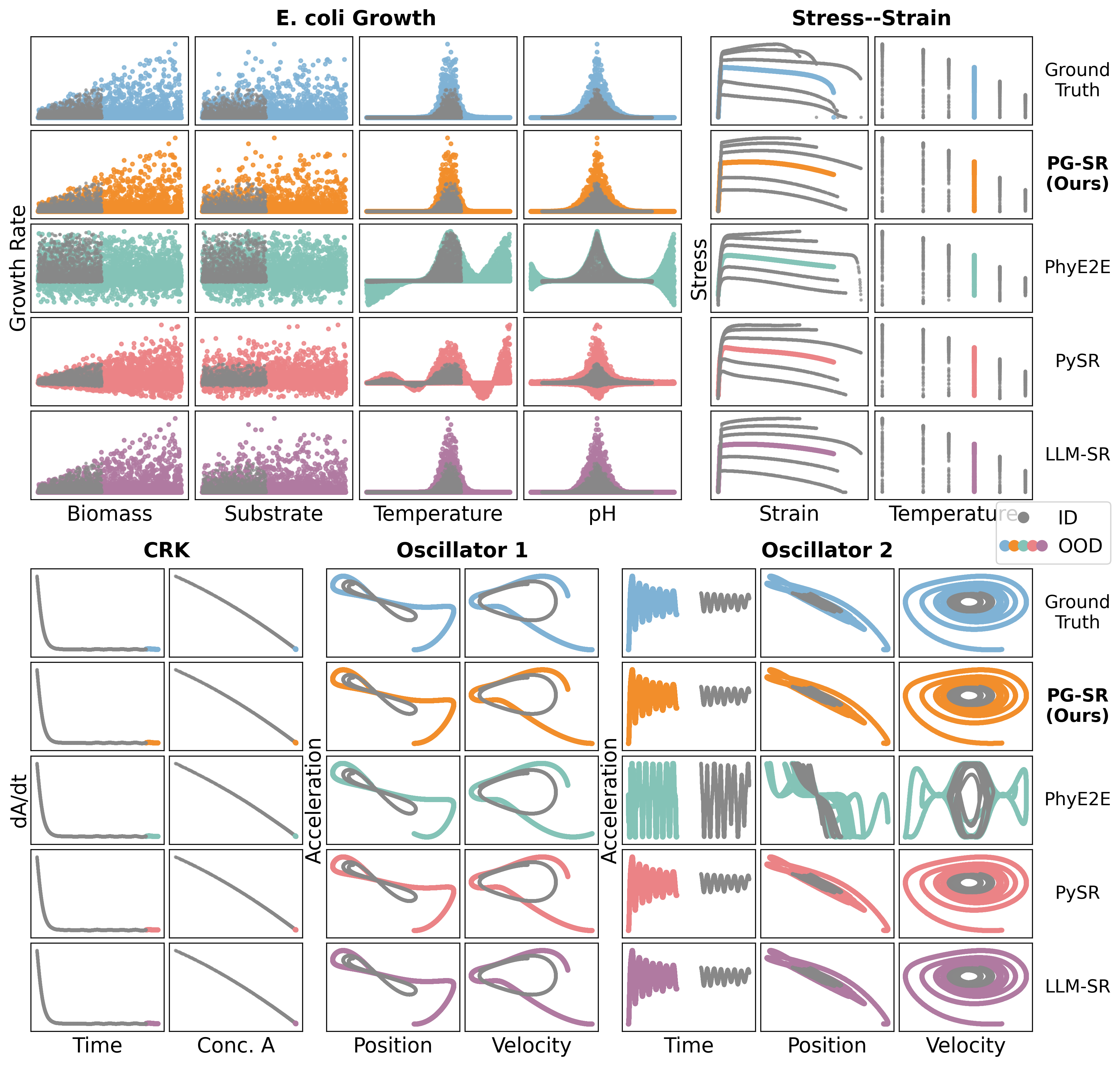}
\caption{Qualitative comparison of predictive trajectories for PG-SR and baselines on ID (gray) and OOD (colored) regions.}
	\label{fig:qualitative_examples}
\end{figure*} 

\subsubsection{Stage II: Prior-guided Evolution}\label{sec:evolution}
In this stage, equations are iteratively evolved using parallel islands to prevent premature convergence.
A  prompt is constructed for the LLM by randomly selecting an island and applying a two-level sampling strategy
to choose skeletons. Within each island, candidate skeletons with the same $S_{\text{MSE}}$ score are grouped into a cluster. Clusters are then sampled according to their PACE scores using a softmax distribution:
\begin{equation}
P(\text{cluster}_i) = \frac{\exp(S_i / \tau)}{\sum_{j=1}^{K} \exp(S_j / \tau)} \,,
\end{equation}
Where $S_i$ is the PACE score of cluster $i$, $K$ is the number of clusters, and $\tau$ is the temperature.
Within each cluster, skeletons are sampled based on length, with shorter programs favored. If a cluster contains valid skeletons $(V=1)$, sampling is done only from them; otherwise, all skeletons are sampled using the same strategy. The prompt combines the problem description, exemplar skeletons from the equation pool, and insights retrieved from the insight pool, guiding the LLM to generate better candidates.
After parsing the generated equations, parameters are optimized via L-BFGS-B to minimize MSE. Candidates are validated against prior constraints with a stochastic retry mechanism: if a candidate fails after optimization, its parameters are re-initialized and re-optimized up to 10 times. Then, they are scored using $S_{\text{MSE}}$ and added to $E$ to guide evolution. The evolution continues until the maximum iterations are reached.

\subsubsection{Stage III: Residual-Enhanced Refinement with Reflection Mechanism}\label{sec:refine}
To further refine equation skeletons, we implement a residual-enhanced refinement with a reflection mechanism, periodically applied to the best candidates from each island. For each island optimum, the process begins by analyzing its residual vector using statistical metrics (NMSE, bias, skewness, and kurtosis), SHAP values, and decision-tree-based high-error region analysis to identify defect variables and feature subspaces responsible for major fitting errors. For structural guidance, we retrieve reference candidates from $E$ whose error profiles $\mathbf{r}_i$ are aligned with the current residual $\mathbf{r}$, measured by cosine similarity $\text{cosine\_sim}(\mathbf{r}, \mathbf{r}_i) = \mathbf{r}^{\top} \mathbf{r}_i / (\|\mathbf{r}\| \|\mathbf{r}_i\|)$.
For each top-ranked reference candidate, the system retrieves relevant information from the Insight Pool to construct an LLM prompt. If prior refinement histories exist, the prompt includes (i) the original equation, (ii) the improved equation, and (iii) a natural language explanation of the structural modification; otherwise, only the raw equation is provided. After each refinement, the system triggers a reflection mechanism: when a performance improvement is achieved, successful experiences are summarized; otherwise, failure cases are analyzed. The resulting insights are added to the Insight Pool to guide skeleton evolution within the same island, while successful experiences can be retrieved by other islands in subsequent refinements based on residual vector similarity.

\section{Experimental Results}
\subsection{Dataset}
We evaluate our method on datasets from biology, materials, chemistry, and physics drawn from LLM-SRBench~\cite{shojaee2025llm}, to avoid LLM memorization.

\textbf{E. coli Growth:} \textit{E. coli} growth dynamics under varying substrate, temperature, and pH conditions.

\textbf{Material Stress–Strain:} Experimental stress–strain measurements of Aluminum 6061-T651 collected across different temperature settings.

\textbf{Chemical Reaction Kinetics:} A chemical reaction system describing the temporal evolution of species concentrations.

\textbf{Nonlinear Oscillator 1:} A time-invariant nonlinear oscillator whose dynamics depend solely on the system states.

\textbf{Nonlinear Oscillator 2:} A time-dependent nonlinear oscillator driven by explicit temporal forcing.

Further details on the datasets are provided in Appendix~\ref{app:dataset}.

\subsection{Baselines and PG-SR Configuration}
We compare PG-SR with state-of-the-art symbolic regression approaches across three categories.
Search-based approaches include GPLearn~\cite{mccormick2019gplearn}, PySR~\cite{cranmer2023interpretable},
DSR~\cite{petersen2019deep}, uDSR~\cite{landajuela2022unified}, Operon~\cite{burlacu2020operon},
and GP-GOMEA~\cite{virgolin2021improving}.
Transformer-based approaches include TPSR~\cite{shojaee2023transformer}, E2E~\cite{kamienny2022end},
NeSymReS~\cite{biggio2021neural}, and PhyE2E~\cite{ying2025neural}.
LLM-based approaches include LLM-SR~\cite{shojaee2024llm}, LaSR~\cite{grayeli2024symbolic},
and DrSR~\cite{wang2025drsr}.
All experiments involving LLM-based approaches are conducted using GPT-4o-mini~\cite{openai_gpt4o_mini}. Implementation details are provided in Appendix~\ref{app:baselines} and Appendix~\ref{app:pg-config}.
For all approaches, the reported experimental results correspond to the best performance across multiple runs.

\begin{figure}[t]
    \centering
    \includegraphics[width=\linewidth]{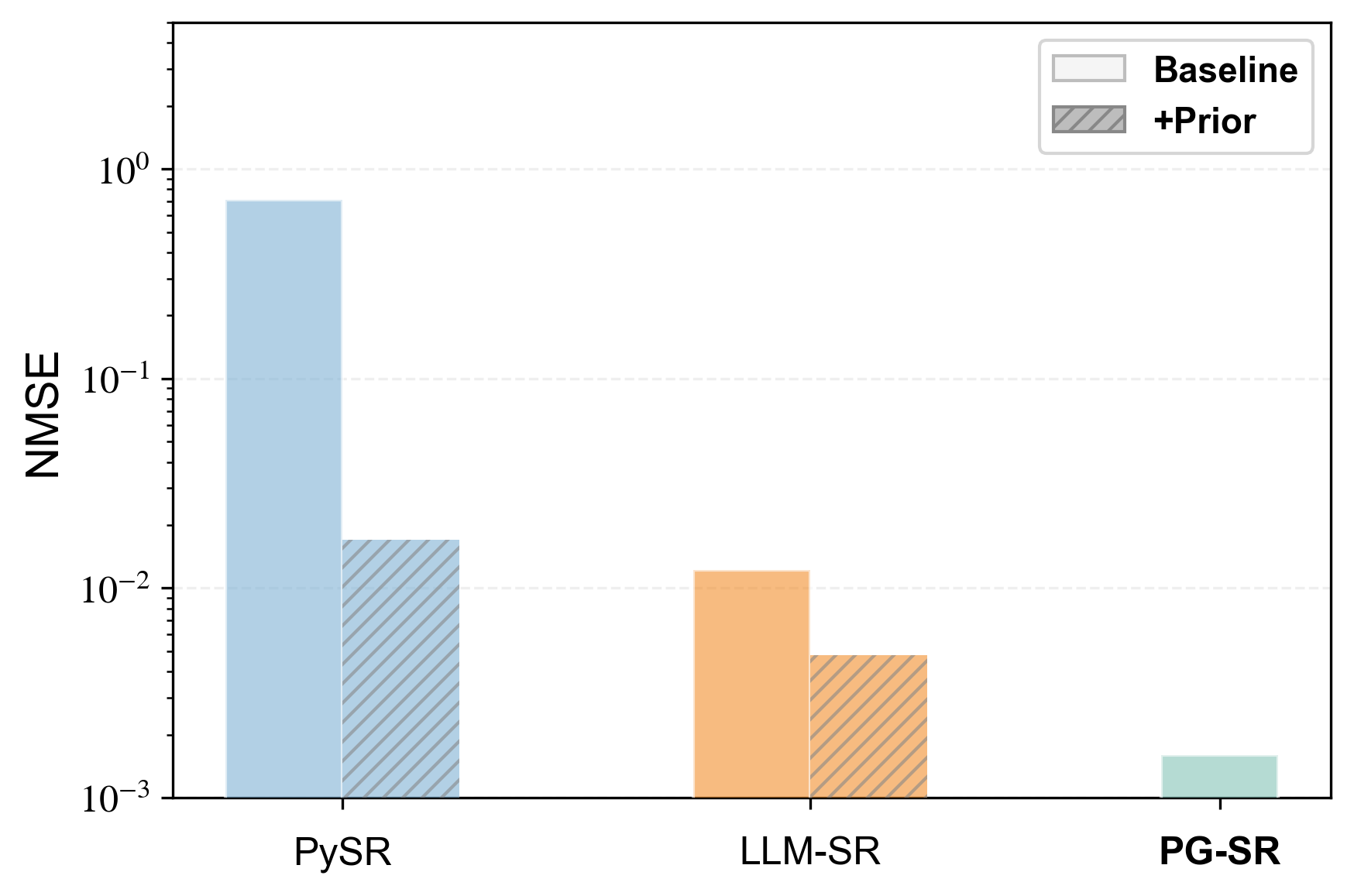}
    \caption{OOD generalization of PySR and LLM-SR with and without prior augmentation, compared to PG-SR.}
    \label{FIG:prior_augmentation}
\end{figure}

\subsection{Quantitative Results}
\label{Quantitative results}
We evaluate PG-SR using both ID and OOD data. OOD performance reflects whether the discovered equations adhere to underlying scientific principles, making it a key metric for scientific discovery. We use the normalized mean squared error (NMSE) as the evaluation metric, defined as
\begin{equation}
\mathrm{NMSE} = \frac{\frac{1}{N}\sum_{i=1}^{N}\left(y_i - \hat{y}_i\right)^2}{\frac{1}{N}\sum_{i=1}^{N}\left(y_i - \bar{y}\right)^2}.
\end{equation}
As shown in Table~\ref{tab:sr_baseline_comparison}, PG-SR achieves the lowest NMSE on both ID and OOD data across all datasets, demonstrating strong predictive accuracy and scientific consistency. 
Given that PG-SR relies on LLMs as a generation engine, it is important to assess whether this performance is sensitive to the backbone model. Therefore, we evaluate PG-SR with another LLM backbone, Llama-3.3-70B~\cite{dubey2024llama}. Results across most datasets indicate comparable performance across backbones, suggesting that PG-SR is relatively robust to the choice of LLM backbone. Besides, Appendix~\ref{app:discovered_equations} provides the equations discovered by PG-SR across all datasets, with the Oscillator 2 case showing an almost exact recovery of the ground-truth  equation.

\subsection{Qualitative Results}
For a more intuitive comparison, Figure~\ref{fig:qualitative_examples} visualizes the predictive trajectories, illustrating the behavior of each method.
As shown in Figure~\ref{fig:qualitative_examples}, within the ID regime (gray dots), PG-SR and all baseline approaches closely match the ground truth trajectories, indicating that most approaches can achieve accurate interpolation. However, in the OOD regions (colored dots), PG-SR exhibits  better alignment with the ground truth dynamics across diverse systems, while baseline approaches often deviate or diverge. Overall, PG-SR maintains scientific consistency with the underlying  equations,  following ground truth trajectories in both ID and OOD regions.

\subsection{Impact of Explicit Prior Constraints}
To investigate whether explicit prior constraints can mitigate pseudo-equations, we evaluated other representative baseline approaches, including PySR (search-based) and LLM-SR (LLM-based), with explicit prior constraints incorporated.
For a controlled comparison, both PySR and LLM-SR additionally leverage the prior constraint checker and the PACE mechanism from PG-SR for candidate evaluation, while preserving their original search and optimization procedures; specifically, PySR incorporates these components via a customized loss function, whereas LLM-SR adopts the same  strategy as PG-SR.
In contrast, transformer-based approaches require retraining on large-scale synthetic data to incorporate priors and are unsuitable for comparison.
As shown in Figure~\ref{FIG:prior_augmentation}, prior-augmented PySR and LLM-SR achieve better OOD generalization than their vanilla versions, demonstrating that explicit prior constraints can effectively improve scientific consistency. 
Nevertheless, their performance remains inferior to PG-SR, demonstrating the superiority of the PG-SR framework.

\begin{figure}[t]
    \centering
    \includegraphics[width=\linewidth]{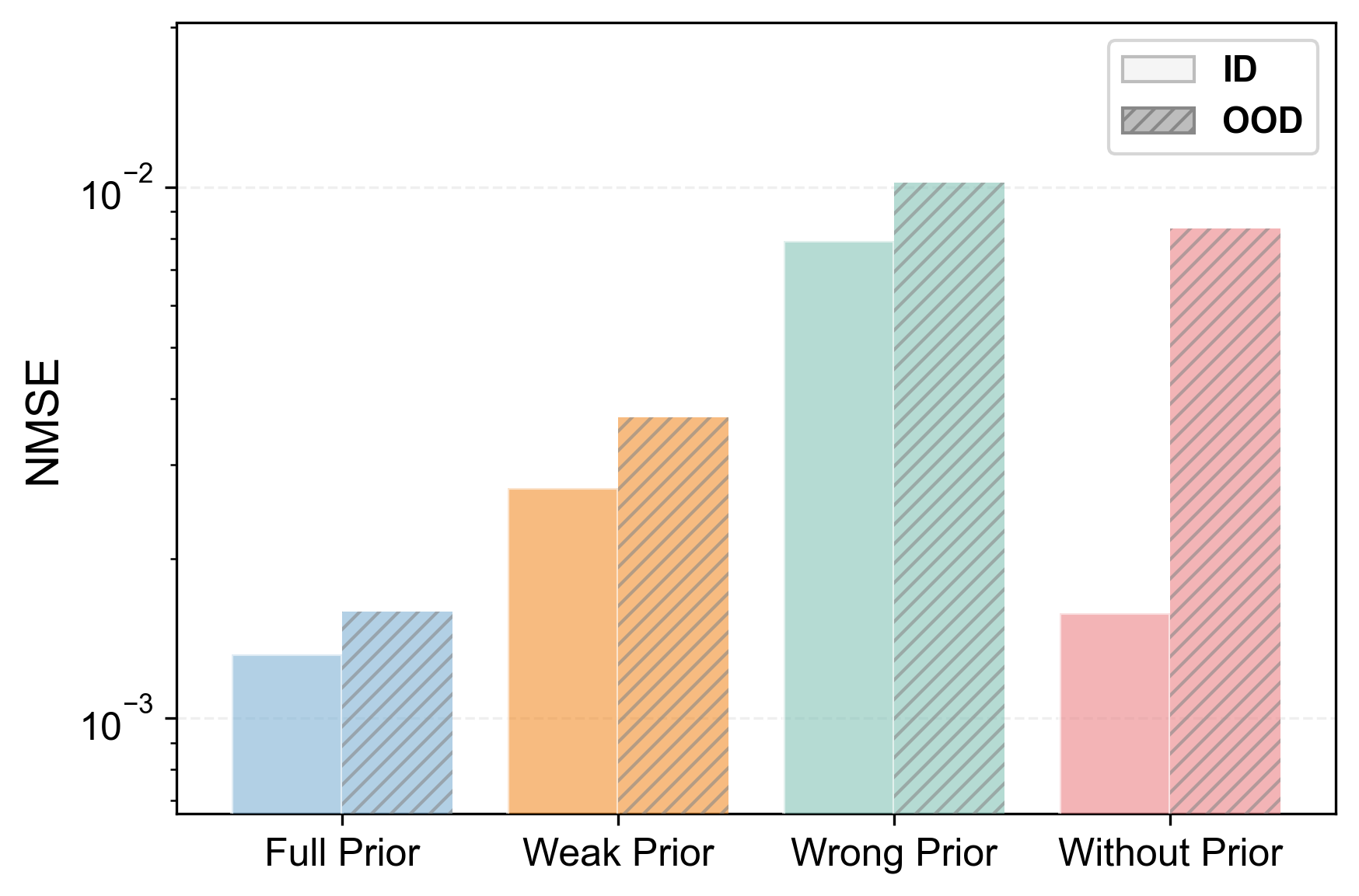}
    \caption{PG-SR performance under different prior quality settings: No Prior, Weak Priors, Wrong Priors, and Full Priors.}
    \label{FIG:prior_quality}
\end{figure}

\subsection{Sensitivity Analysis on Prior Quality}
Since prior knowledge may be incomplete or even misleading, we further assessed the sensitivity of PG-SR to prior quality under four settings on the E. coli Growth dataset: No Prior, Weak Priors (specifying only optimal temperature 37$^\circ$C and pH~7), 
Wrong Priors (incorrect values, e.g., optimal temperature 15$^\circ$C and pH~4), 
and Full Priors (details in Appendix~\ref{app:priors}).
As shown in Figure~\ref{FIG:prior_quality}, Full Priors yield the best performance, while Weak Priors provide partial improvements. 
Although wrong priors degrade performance, the results remain comparable to the no-prior setting, indicating that PG-SR does not collapse under misleading prior information. This robustness is enabled by PACE, which moderates the influence of priors by allowing observed data to retain influence during evolution.

\subsection{Ablation Study }
To evaluate the contribution of each component in PG-SR, we conduct ablation studies on the E. coli Growth dataset, with results summarized in Table~\ref{tab:ablation}.
Removing the warm-up stage (w/o Warmup) results in low-quality equation skeletons in the experience pool during the initial stage, leading to poor initialization and adversely affecting subsequent generation, thereby degrading overall performance within the same number of iterations.
Removing the PACE mechanism (w/o PACE) leads to the direct discarding of equation skeletons that violate prior constraints. This introduces discontinuities into the search landscape, which restricts search diversity and hinders the optimization process in the early stages, ultimately resulting in performance degradation.
Removing the refinement stage (w/o Refine) makes PG-SR more prone to local optima that satisfy the constraints but exhibit poor data fitting, highlighting the role of the refinement mechanism in guiding the search out of such regions.
Overall, these ablation results indicate that each component effectively contributes to the performance of PG-SR.
\begin{table}[t]
\centering
\small
\caption{Ablation study.}
\label{tab:ablation}
\setlength{\tabcolsep}{6pt}
\renewcommand{\arraystretch}{1}
\begin{tabular}{lcc}
\toprule
\textbf{Variant} &
\textbf{ID NMSE} $\downarrow$ &
\textbf{OOD NMSE} $\downarrow$ \\
\midrule
\rowcolor[gray]{0.95} \textbf{Full Variant} & \textbf{1.32e-3} & \textbf{1.59e-3} \\
w/o Warmup  & 5.70e-3 & 5.09e-2 \\
w/o PACE    & 1.28e-3 & 2.98e-1 \\
w/o Refine  & 8.75e-2 & 1.91e-1 \\
\bottomrule
\end{tabular}
\end{table}
\begin{figure}[b]
	\centering
    	\includegraphics[width=0.45\textwidth]{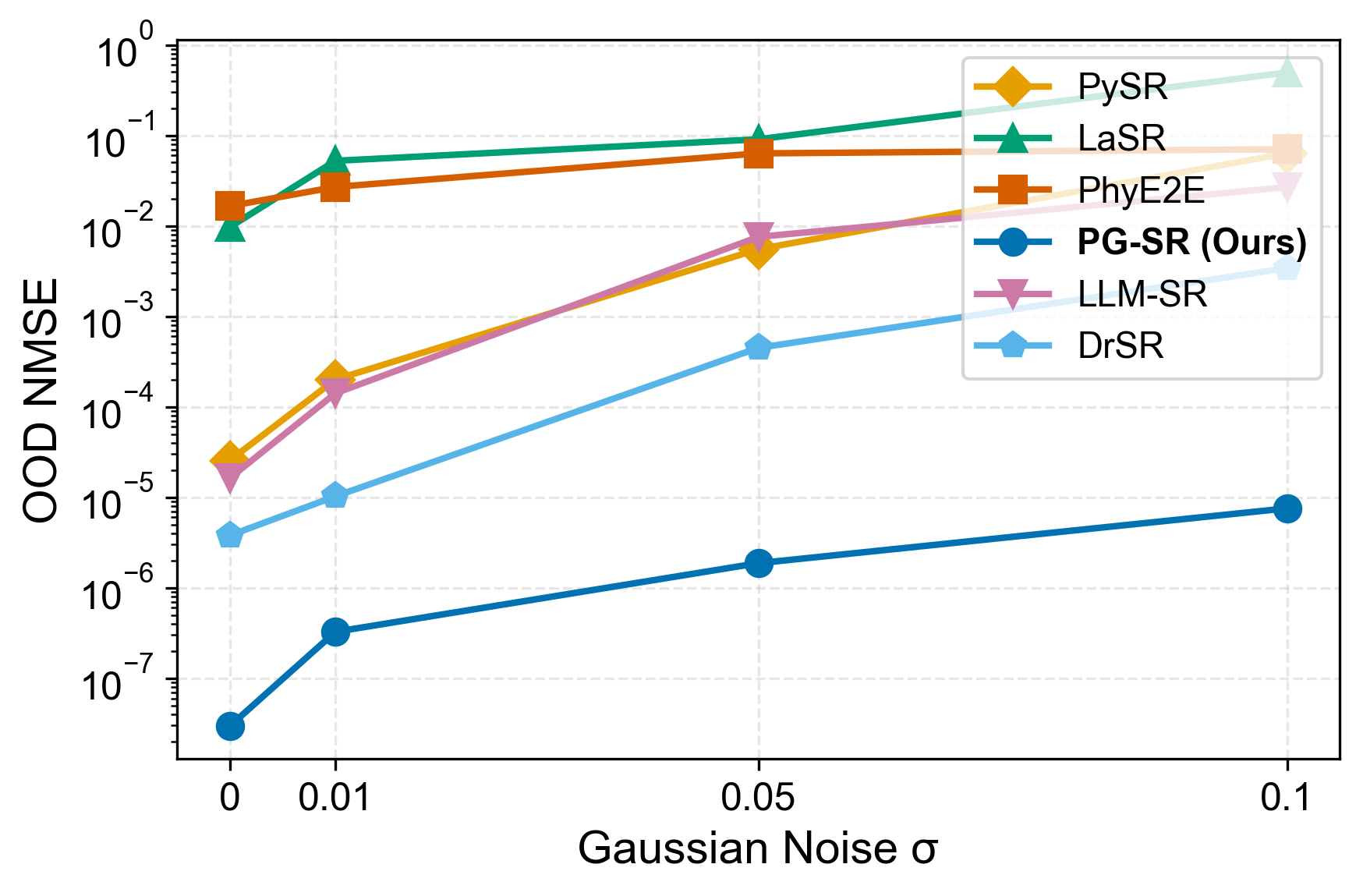}
\caption{Noise robustness analysis. OOD NMSE under different noise levels ($\sigma$) for PG-SR and baseline models.}
	\label{FIG:noise}
\end{figure} 
\subsection{Robustness to Noisy Data}
In real-world scenarios, data are often corrupted by noise, posing a challenge to existing approaches grounded in empirical risk minimization. To investigate whether PG-SR remains robust under such conditions, we add Gaussian perturbations to the Oscillator~1 dataset:
\begin{equation}
\tilde{x} = x + \epsilon, \quad \epsilon \sim \mathcal{N}(0, \sigma^2), \quad \sigma \in \{0.01, 0.05, 0.1\}.
\end{equation}
Since the training data are corrupted by noise, prior checking in this setting is performed in a statistical manner rather than through pointwise value matching.
As shown in Figure~\ref{FIG:noise}, as the noise level increases, baseline approaches tend to overfit noisy observations, leading to poor OOD generalization. 
While PG-SR is also slightly affected by noise, it maintains substantially better robustness than baselines due to the guidance provided by explicit prior constraints.

\subsection{Robustness to Data Scarcity}
\label{sec:data_sparse}
In real-world settings, data acquisition is often costly or constrained, making discovery under limited data essential. To evaluate PG-SR in data-scarce regimes and validate Proposition~\ref{prop:complexity}, we analyze performance from 100\% down to 5\% data availability.
As shown in Figure~\ref{fig:data_sparse}, PG-SR exhibits a slight NMSE increase when the data availability is reduced from 100\% to 50\%.
In this regime, the reduced data make it harder to distinguish between different  skeletons that satisfy the priors, leading to a temporary performance degradation.
When the data are further reduced from 50\% to 5\%, the discovery process becomes dominated by the  priors, which strongly restrict the hypothesis space, leaving the remaining candidate solutions  constrained in structure and thus making performance insensitive to further changes in data scale.

\begin{figure}[t]
    \centering
    	\includegraphics[width=0.45\textwidth]{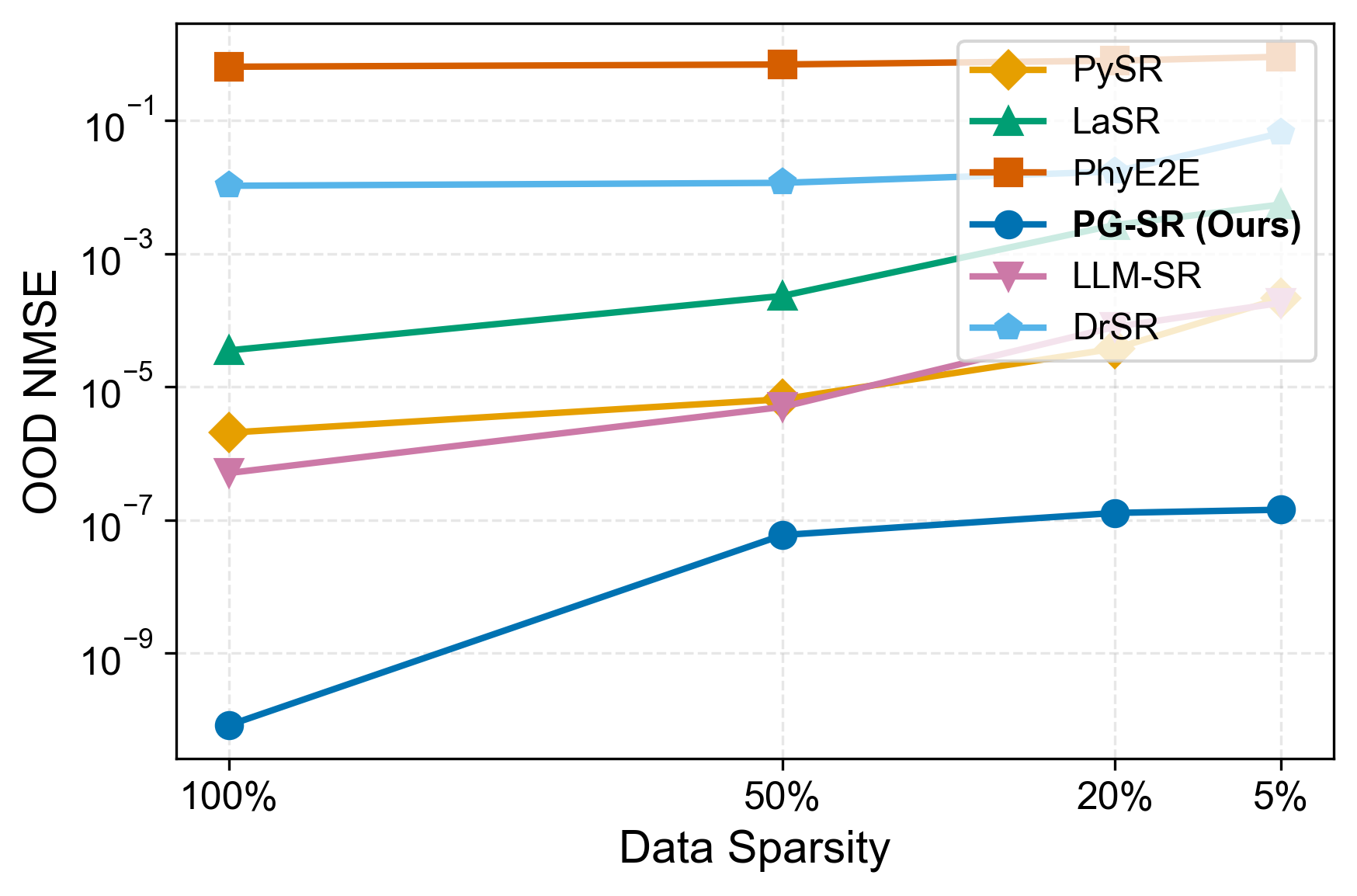}
\caption{Data Scarcity Analysis. OOD NMSE under different data availability settings for PG-SR and baseline models.}
    \label{fig:data_sparse}
\end{figure}

\section{Conclusion}
This paper proposes PG-SR to address the Pseudo-Equation Trap in symbolic regression, taking a step toward scientific consistency in equation discovery. Experiments show that PG-SR outperforms SOTA baselines across diverse domains.
Despite its effectiveness, PG-SR still relies on manual intervention when constructing executable prior constraints. Future work will explore automating prior constraint generation by fine-tuning LLMs to synthesize executable constraint programs from problem descriptions and data statistics, moving PG-SR toward a fully autonomous framework.

\section*{Broader Impact Statement}
This paper presents work whose goal is to advance the field of Machine Learning. 
There are many potential societal consequences of our work, none which we feel must be specifically highlighted here.

\bibliographystyle{icml2026}
\bibliography{icml2026}

@article{petersen2019deep,
  title={Deep symbolic regression: Recovering mathematical expressions from data via risk-seeking policy gradients},
  author={Petersen, Brenden K and Landajuela, Mikel and Mundhenk, T Nathan and Santiago, Claudio P and Kim, Soo K and Kim, Joanne T},
  journal={arXiv preprint arXiv:1912.04871},
  year={2019}
}

@article{cranmer2023interpretable,
  title={Interpretable machine learning for science with PySR and SymbolicRegression. jl},
  author={Cranmer, Miles},
  journal={arXiv preprint arXiv:2305.01582},
  year={2023}
}

@article{landajuela2022unified,
  title={A unified framework for deep symbolic regression},
  author={Landajuela, Mikel and Lee, Chak Shing and Yang, Jiachen and Glatt, Ruben and Santiago, Claudio P and Aravena, Ignacio and Mundhenk, Terrell and Mulcahy, Garrett and Petersen, Brenden K},
  journal={Advances in Neural Information Processing Systems},
  volume={35},
  pages={33985--33998},
  year={2022}
}

@article{virgolin2021improving,
  title={Improving model-based genetic programming for symbolic regression of small expressions},
  author={Virgolin, Marco and Alderliesten, Tanja and Witteveen, Cees and Bosman, Peter AN},
  journal={Evolutionary computation},
  volume={29},
  number={2},
  pages={211--237},
  year={2021},
  publisher={MIT Press One Rogers Street, Cambridge, MA 02142-1209, USA journals-info~...}
}

@misc{mccormick2019gplearn,
  author       = {McCormick, Trevor},
  title        = {gplearn: Genetic Programming in Python},
  year         = {2019},
  howpublished = {\url{https://github.com/trevorstephens/gplearn}},
  note         = {GitHub repository}
}

@inproceedings{burlacu2020operon,
  title={Operon C++ an efficient genetic programming framework for symbolic regression},
  author={Burlacu, Bogdan and Kronberger, Gabriel and Kommenda, Michael},
  booktitle={Proceedings of the 2020 genetic and evolutionary computation conference companion},
  pages={1562--1570},
  year={2020}
}

@article{shojaee2023transformer,
  title={Transformer-based planning for symbolic regression},
  author={Shojaee, Parshin and Meidani, Kazem and Barati Farimani, Amir and Reddy, Chandan},
  journal={Advances in Neural Information Processing Systems},
  volume={36},
  pages={45907--45919},
  year={2023}
}

@article{kamienny2022end,
  title={End-to-end symbolic regression with transformers},
  author={Kamienny, Pierre-Alexandre and d'Ascoli, St{\'e}phane and Lample, Guillaume and Charton, Fran{\c{c}}ois},
  journal={Advances in Neural Information Processing Systems},
  volume={35},
  pages={10269--10281},
  year={2022}
}

@inproceedings{biggio2021neural,
  title={Neural symbolic regression that scales},
  author={Biggio, Luca and Bendinelli, Tommaso and Neitz, Alexander and Lucchi, Aurelien and Parascandolo, Giambattista},
  booktitle={International Conference on Machine Learning},
  pages={936--945},
  year={2021},
  organization={Pmlr}
}

@article{grayeli2024symbolic,
  title={Symbolic regression with a learned concept library},
  author={Grayeli, Arya and Sehgal, Atharva and Costilla Reyes, Omar and Cranmer, Miles and Chaudhuri, Swarat},
  journal={Advances in Neural Information Processing Systems},
  volume={37},
  pages={44678--44709},
  year={2024}
}

@article{shojaee2024llm,
  title={Llm-sr: Scientific equation discovery via programming with large language models},
  author={Shojaee, Parshin and Meidani, Kazem and Gupta, Shashank and Farimani, Amir Barati and Reddy, Chandan K},
  journal={arXiv preprint arXiv:2404.18400},
  year={2024}
}

@article{wang2025drsr,
  title={DrSR: LLM based Scientific Equation Discovery with Dual Reasoning from Data and Experience},
  author={Wang, Runxiang and Wang, Boxiao and Li, Kai and Zhang, Yifan and Cheng, Jian},
  journal={arXiv preprint arXiv:2506.04282},
  year={2025}
}

@article{ying2025neural,
  title={A neural symbolic model for space physics},
  author={Ying, Jie and Lin, Haowei and Yue, Chao and Chen, Yajie and Xiao, Chao and Shi, Quanqi and Liang, Yitao and Yau, Shing-Tung and Zhou, Yuan and Ma, Jianzhu},
  journal={Nature Machine Intelligence},
  pages={1--16},
  year={2025},
  publisher={Nature Publishing Group UK London}
}

@article{koza1994genetic,
  title={Genetic programming as a means for programming computers by natural selection},
  author={Koza, John R},
  journal={Statistics and computing},
  volume={4},
  number={2},
  pages={87--112},
  year={1994},
  publisher={Springer}
}

@article{makke2024interpretable,
  title={Interpretable scientific discovery with symbolic regression: a review},
  author={Makke, Nour and Chawla, Sanjay},
  journal={Artificial Intelligence Review},
  volume={57},
  number={1},
  pages={2},
  year={2024},
  publisher={Springer}
}

@article{shojaee2025llm,
  title={LLM-SRBench: A New Benchmark for Scientific Equation Discovery with Large Language Models},
  author={Shojaee, Parshin and Nguyen, Ngoc-Hieu and Meidani, Kazem and Farimani, Amir Barati and Doan, Khoa D and Reddy, Chandan K},
  journal={arXiv preprint arXiv:2504.10415},
  year={2025}
}

@book{mohri2018foundations,
  title={Foundations of machine learning},
  author={Mohri, Mehryar and Rostamizadeh, Afshin and Talwalkar, Ameet},
  year={2018},
  publisher={MIT press}
}

@inproceedings{kocsis2006bandit,
  title={Bandit based monte-carlo planning},
  author={Kocsis, Levente and Szepesv{\'a}ri, Csaba},
  booktitle={European conference on machine learning},
  pages={282--293},
  year={2006},
  organization={Springer}
}

@article{brown2020language,
  title={Language models are few-shot learners},
  author={Brown, Tom and Mann, Benjamin and Ryder, Nick and Subbiah, Melanie and Kaplan, Jared D and Dhariwal, Prafulla and Neelakantan, Arvind and Shyam, Pranav and Sastry, Girish and Askell, Amanda and others},
  journal={Advances in neural information processing systems},
  volume={33},
  pages={1877--1901},
  year={2020}
}

@article{arjovsky2019invariant,
  title={Invariant risk minimization},
  author={Arjovsky, Martin and Bottou, L{\'e}on and Gulrajani, Ishaan and Lopez-Paz, David},
  journal={arXiv preprint arXiv:1907.02893},
  year={2019}
}

@article{karniadakis2021physics,
  title={Physics-informed machine learning},
  author={Karniadakis, George Em and Kevrekidis, Ioannis G and Lu, Lu and Perdikaris, Paris and Wang, Sifan and Yang, Liu},
  journal={Nature Reviews Physics},
  volume={3},
  number={6},
  pages={422--440},
  year={2021},
  publisher={Nature Publishing Group UK London}
}

@article{schmidt2009distilling,
  title={Distilling free-form natural laws from experimental data},
  author={Schmidt, Michael and Lipson, Hod},
  journal={science},
  volume={324},
  number={5923},
  pages={81--85},
  year={2009},
  publisher={American Association for the Advancement of Science}
}

@article{udrescu2020ai,
  title={AI Feynman: A physics-inspired method for symbolic regression},
  author={Udrescu, Silviu-Marian and Tegmark, Max},
  journal={Science advances},
  volume={6},
  number={16},
  pages={eaay2631},
  year={2020},
  publisher={American Association for the Advancement of Science}
}

@article{wang2019symbolic,
  title={Symbolic regression in materials science},
  author={Wang, Yiqun and Wagner, Nicholas and Rondinelli, James M},
  journal={MRS communications},
  volume={9},
  number={3},
  pages={793--805},
  year={2019},
  publisher={Cambridge University Press}
}

@article{cranmer2020discovering,
  title={Discovering symbolic models from deep learning with inductive biases},
  author={Cranmer, Miles and Sanchez Gonzalez, Alvaro and Battaglia, Peter and Xu, Rui and Cranmer, Kyle and Spergel, David and Ho, Shirley},
  journal={Advances in neural information processing systems},
  volume={33},
  pages={17429--17442},
  year={2020}
}

@article{vastl2024symformer,
  title={Symformer: End-to-end symbolic regression using transformer-based architecture},
  author={Vastl, Martin and Kulh{\'a}nek, Jon{\'a}{\v{s}} and Kubal{\'\i}k, Ji{\v{r}}{\'\i} and Derner, Erik and Babu{\v{s}}ka, Robert},
  journal={IEEE Access},
  volume={12},
  pages={37840--37849},
  year={2024},
  publisher={IEEE}
}

@article{romera2024mathematical,
  title={Mathematical discoveries from program search with large language models},
  author={Romera-Paredes, Bernardino and Barekatain, Mohammadamin and Novikov, Alexander and Balog, Matej and Kumar, M Pawan and Dupont, Emilien and Ruiz, Francisco JR and Ellenberg, Jordan S and Wang, Pengming and Fawzi, Omar and others},
  journal={Nature},
  volume={625},
  number={7995},
  pages={468--475},
  year={2024},
  publisher={Nature Publishing Group UK London}
}

@article{ji2023survey,
  title={Survey of hallucination in natural language generation},
  author={Ji, Ziwei and Lee, Nayeon and Frieske, Rita and Yu, Tiezheng and Su, Dan and Xu, Yan and Ishii, Etsuko and Bang, Ye Jin and Madotto, Andrea and Fung, Pascale},
  journal={ACM computing surveys},
  volume={55},
  number={12},
  pages={1--38},
  year={2023},
  publisher={ACM New York, NY}
}

@article{chen2021evaluating,
  title={Evaluating large language models trained on code},
  author={Chen, Mark},
  journal={arXiv preprint arXiv:2107.03374},
  year={2021}
}

@article{dubey2024llama,
  title={The llama 3 herd of models},
  author={Dubey, Abhimanyu and Jauhri, Abhinav and Pandey, Abhinav and Kadian, Abhishek and Al-Dahle, Ahmad and Letman, Aiesha and Mathur, Akhil and Schelten, Alan and Yang, Amy and Fan, Angela and others},
  journal={arXiv preprint arXiv:2407.21783},
  year={2024}
}

@misc{openai_gpt4o_mini,
  title       = {GPT‑4o mini: advancing cost‑efficient intelligence},
  author      = {{OpenAI}},
  year        = {2024},
  howpublished={\url{https://openai.com/is-IS/index/gpt-4o-mini-advancing-cost-efficient-intelligence/}}
}

@article{de2024ai,
  title={AI-Lorenz: A physics-data-driven framework for Black-Box and Gray-Box identification of chaotic systems with symbolic regression},
  author={De Florio, Mario and Kevrekidis, Ioannis G and Karniadakis, George Em},
  journal={Chaos, Solitons \& Fractals},
  volume={188},
  pages={115538},
  year={2024},
  publisher={Elsevier}
}

@article{achiam2023gpt,
  title={Gpt-4 technical report},
  author={Achiam, Josh and Adler, Steven and Agarwal, Sandhini and Ahmad, Lama and Akkaya, Ilge and Aleman, Florencia Leoni and Almeida, Diogo and Altenschmidt, Janko and Altman, Sam and Anadkat, Shyamal and others},
  journal={arXiv preprint arXiv:2303.08774},
  year={2023}
}

@article{team2023gemini,
  title={Gemini: a family of highly capable multimodal models},
  author={Team, Gemini and Anil, Rohan and Borgeaud, Sebastian and Alayrac, Jean-Baptiste and Yu, Jiahui and Soricut, Radu and Schalkwyk, Johan and Dai, Andrew M and Hauth, Anja and Millican, Katie and others},
  journal={arXiv preprint arXiv:2312.11805},
  year={2023}
}

@article{whitley1999island,
  title={The island model genetic algorithm: On separability, population size and convergence},
  author={Whitley, Darrell and Rana, Soraya and Heckendorn, Robert B},
  journal={Journal of computing and information technology},
  volume={7},
  number={1},
  pages={33--47},
  year={1999},
  publisher={Sveu{\v{c}}ili{\v{s}}te u Zagrebu Sveu{\v{c}}ili{\v{s}}ni ra{\v{c}}unski centar}
}

@article{byrd1995limited,
  title={A limited memory algorithm for bound constrained optimization},
  author={Byrd, Richard H and Lu, Peihuang and Nocedal, Jorge and Zhu, Ciyou},
  journal={SIAM Journal on scientific computing},
  volume={16},
  number={5},
  pages={1190--1208},
  year={1995},
  publisher={SIAM}
}

@article{uc2023survey,
  title={Survey on reinforcement learning for language processing},
  author={Uc-Cetina, Victor and Navarro-Guerrero, Nicol{\'a}s and Martin-Gonzalez, Anabel and Weber, Cornelius and Wermter, Stefan},
  journal={Artificial Intelligence Review},
  volume={56},
  number={2},
  pages={1543--1575},
  year={2023},
  publisher={Springer}
}

@article{talagrand1996new,
  title={New concentration inequalities in product spaces},
  author={Talagrand, Michel},
  journal={Inventiones mathematicae},
  volume={126},
  number={3},
  pages={505--563},
  year={1996},
  publisher={Springer}
}

@article{bartlett2002rademacher,
  title={Rademacher and gaussian complexities: Risk bounds and structural results},
  author={Bartlett, Peter L and Mendelson, Shahar},
  journal={Journal of machine learning research},
  volume={3},
  number={Nov},
  pages={463--482},
  year={2002}
}

@article{mcdiarmid1989method,
  title={On the method of bounded differences},
  author={McDiarmid, Colin and others},
  journal={Surveys in combinatorics},
  volume={141},
  number={1},
  pages={148--188},
  year={1989},
  publisher={Norwich}
}

\newpage
\appendix
\onecolumn

\section{Theoretical Analysis}
\label{sec:appendix_proofs}
In this appendix, we provide detailed proofs for the theoretical results presented in Section~\ref{sec:theory}, based on the statistical learning theory framework~\cite{mohri2018foundations}. Let $\mathcal{P}$ denote an unknown joint distribution over $\mathcal{X} \times \mathcal{Y}$. The training dataset $\mathcal{D} = \{(x_i, y_i)\}_{i=1}^{N}$ consists of $N$ i.i.d.\ samples drawn from $\mathcal{P}$.
\subsection{Preliminaries and Assumptions}

To ensure the theoretical validity of the generalization bounds in the context of symbolic regression, we adopt the following standard regularity assumption regarding the hypothesis space and loss function.

\begin{assumption}[Boundedness] \label{assump:boundedness}
Let $\mathcal{X} \subset \mathbb{R}^d$ be the compact input domain and $\mathcal{Y} \subset \mathbb{R}$ be the bounded output space. We assume that the hypothesis space $\mathcal{H}$ consists of functions that are bounded within the domain. Specifically, there exists a constant $B > 0$ such that for any candidate function $f \in \mathcal{H}$ and any $x \in \mathcal{X}$, $|f(x)| \le B$. Consequently, assuming the target $y$ is also bounded, the loss function $\ell(f(x), y)$ is bounded by a constant $M < \infty$ and is Lipschitz continuous with respect to the prediction.
\end{assumption}

\textit{Remark:} In practice, this assumption is enforced by detecting and rejecting candidate equations that exhibit singularities or numerical overflow within the domain of interest, or by applying a truncation operator to the function outputs.

\subsection{Proof of Lemma 3.2 (Generalization Bound)}
\label{sec:proof_thm_3_2}

\textbf{Lemma 3.2 (Generalization Bound).} 
\textit{Let $\mathcal{H}$ be the hypothesis space. 
Assume the loss function $\ell$ is $\lambda$-Lipschitz with respect to its first argument and bounded by $M$. 
For any $\delta > 0$, with probability at least $1-\delta$, the generalization error of any $f \in \mathcal{H}$ is bounded by:}
\begin{equation}
    R(f) \le \hat{R}_N(f) 
    + 2\lambda \Rad_N(\mathcal{H}) 
    + M\sqrt{\frac{\log(1/\delta)}{2N}}.
\end{equation}

\begin{proof}
The proof follows the standard symmetrization argument in statistical learning theory~\cite{mohri2018foundations}. To ensure the hypothesis space $\mathcal{H}$ is well-defined and has finite Rademacher complexity, we consider the set of symbolic expression trees with a maximum depth $L$.

Let $\Phi(\mathcal{D}) = \sup_{f \in \mathcal{H}} (R(f) - \hat{R}_N(f))$. We aim to bound $\Phi(\mathcal{D})$. Since the loss function is bounded by $M$, changing one example $(x_i, y_i)$ to $(x_i', y_i')$ changes $\Phi(\mathcal{D})$ by at most $M/N$. By McDiarmid's inequality~\cite{mcdiarmid1989method}, for any $\delta > 0$, with probability at least $1-\delta$:
\begin{equation}
\Phi(\mathcal{D}) \le \mathbb{E}_{\mathcal{D}}[\Phi(\mathcal{D})] + M\sqrt{\frac{\log(1/\delta)}{2N}}.
\end{equation}

Next, we bound the expectation $\mathbb{E}_{\mathcal{D}}[\Phi(\mathcal{D})]$. We introduce a ghost dataset $\mathcal{D}' = \{(x_i', y_i')\}_{i=1}^N$ drawn from the same distribution $\mathcal{P}$:
\begin{align}
    \mathbb{E}_{\mathcal{D}}[\Phi(\mathcal{D})] &= \mathbb{E}_{\mathcal{D}} \left[ \sup_{f \in \mathcal{H}} \left( \mathbb{E}_{\mathcal{D}'}[\hat{R}_{N}'(f)] - \hat{R}_N(f) \right) \right] \\
    &\le \mathbb{E}_{\mathcal{D}, \mathcal{D}'} \left[ \sup_{f \in \mathcal{H}} (\hat{R}_{N}'(f) - \hat{R}_N(f)) \right] \quad (\text{by Jensen's inequality}) \\
    &= \mathbb{E}_{\mathcal{D}, \mathcal{D}'} \left[ \sup_{f \in \mathcal{H}} \frac{1}{N} \sum_{i=1}^N (\ell(f(x_i'), y_i') - \ell(f(x_i), y_i)) \right].
\end{align}

Since $\mathcal{D}$ and $\mathcal{D}'$ are i.i.d., introducing Rademacher variables $\boldsymbol{\sigma} = (\sigma_1, \dots, \sigma_N)$ where $\sigma_i \in \{-1, +1\}$ uniformly does not change the distribution. Thus:
\begin{align}
    \mathbb{E}_{\mathcal{D}}[\Phi(\mathcal{D})] &\le \mathbb{E}_{\mathcal{D}, \mathcal{D}', \boldsymbol{\sigma}} \left[ \sup_{f \in \mathcal{H}} \frac{1}{N} \sum_{i=1}^N \sigma_i (\ell(f(x_i), y_i) - \ell(f(x_i', y_i')) ) \right] \\
    &\le 2 \mathbb{E}_{\mathcal{D}, \boldsymbol{\sigma}} \left[ \sup_{f \in \mathcal{H}} \frac{1}{N} \sum_{i=1}^N \sigma_i \ell(f(x_i), y_i) \right] \\
    &= 2 \Rad_N(\ell \circ \mathcal{H}).
\end{align}

Assuming the loss function $\ell$ is $\lambda$-Lipschitz (as implied by Assumption \ref{assump:boundedness}), by Talagrand's contraction lemma~\cite{talagrand1996new,bartlett2002rademacher}, $\Rad_N(\ell \circ \mathcal{H}) \le \lambda \Rad_N(\mathcal{H})$. Substituting this back completes the proof.
\end{proof}
\subsection{Proof of Proposition 3.5 (Complexity Reduction)}
\label{sec:proof_prop_3_5}

\textbf{Proposition 3.5.} \textit{Let $\mathcal{C}$ be a set of restrictive prior constraints satisfying $f_{\text{true}} \in \mathcal{H}_{\mathcal{C}}$. The Rademacher complexity satisfies $\Rad_N(\mathcal{H}_{\mathcal{C}}) \le \Rad_N(\mathcal{H})$.}

\begin{proof}
By Definition 3.4, the constrained subspace is $\mathcal{H}_{\mathcal{C}} = \{f \in \mathcal{H} \mid V(f, \mathcal{C})=1\}$, implying $\mathcal{H}_{\mathcal{C}} \subseteq \mathcal{H}$.

The empirical Rademacher complexity is derived as follows:
\begin{align}
    \Rad_N(\mathcal{H}_{\mathcal{C}}) &= \mathbb{E}_{\boldsymbol{\sigma}} \left[ \sup_{f \in \mathcal{H}_{\mathcal{C}}} \frac{1}{N} \sum_{i=1}^N \sigma_i f(x_i) \right] \label{eq:rad_def} \\
    &\le \mathbb{E}_{\boldsymbol{\sigma}} \left[ \sup_{f \in \mathcal{H}} \frac{1}{N} \sum_{i=1}^N \sigma_i f(x_i) \right] \label{eq:rad_ineq} \\
    &= \Rad_N(\mathcal{H}). \nonumber
\end{align}
\end{proof}

\textbf{Remark.} While the theoretical inequality is non-strict ($\le$) without stronger assumptions on $\mathcal{C}$, a strict reduction ($\Rad_N(\mathcal{H}_{\mathcal{C}}) < \Rad_N(\mathcal{H})$) is expected in practice. Consider the set of pseudo-equations $\mathcal{Q} = \mathcal{H} \setminus \mathcal{H}_{\mathcal{C}}$. Functions in $\mathcal{Q}$ typically possess high capacity  and can flip signs to match the random noise $\boldsymbol{\sigma}$. Since $\mathcal{C}$ explicitly prunes these candidates, the effective search space complexity is strictly reduced.

\section{Mechanism Analysis for PACE}
\label{app:pace}

To understand the effectiveness of the PACE strategy, we analyze its impact on the search landscape continuity and the dynamic trade-off between exploration and constraint satisfaction.

\subsection{Gradient Preservation in Search Landscape}
Standard constrained optimization often employs a hard-filtering baseline, where invalid candidates ($f \notin \mathcal{H}_{\mathcal{C}}$) are assigned a worst-case score (e.g., $-\infty$) or discarded. This creates a discontinuous fitness landscape with zero gradients in the prior-violating regions, causing the search to stall if the evolutionary path requires traversing scientifically inconsistent intermediate states. Here, the term gradient is used only to describe directional signals in the score-based search, rather than gradients of a differentiable objective.
In contrast, PACE maintains informative selection gradients during the exploration phase ($t < 1$). Recall the total score for invalid candidates: $S(f, t) = \sigma(t) S_{\text{base}}(f) - \delta(t)$. Since $\sigma(t) > 0$ and $\delta(t)$ are time-dependent scalars constant across the population at any generation $t$, the scoring function is an affine transformation of the baseline reward. Consequently, for any two invalid candidates $f_1, f_2$, the ranking order is preserved:
\begin{equation}
S(f_1, t) > S(f_2, t) \iff S_{\text{base}}(f_1) > S_{\text{base}}(f_2).
\end{equation}
This linearity ensures that the optimizer receives informative feedback driven by numerical accuracy even within the invalid subspace. By eliminating the zero-gradient plateaus typical of hard-filtering, PACE allows the algorithm to utilize high-performing but inconsistent structures as evolutionary stepping stones toward valid solutions.

\subsection{Dynamic Exploration-Constraint Trade-off}
To visually illustrate how scientific priors are progressively enforced, Figure~\ref{FIG:8} depicts the dynamic reward boundaries induced by the PACE mechanism. Mathematically, the time-varying parameters $\sigma(t)$ and $\delta(t)$ implement a smooth homotopy from unconstrained exploration to strict constraint enforcement. This creates a mutable failure region for invalid candidates while maintaining a fixed success region for valid ones, effectively managing two distinct risks:

\begin{itemize}
    \item \textbf{Exploration Phase (Low $t$):} 
    With $\phi(t) \approx 0$, the penalty $\delta(t)$ is negligible and the reward width $\sigma(t) \approx 1$. In this stage, the failure region is broad, mimicking the unconstrained baseline. This configuration minimizes the \textit{Search Failure Risk} by preventing the premature rejection of promising but scientifically inconsistent candidates, allowing the search beam to encompass diverse structural prototypes.
    
    \item \textbf{Enforcement Phase (High $t$):} 
    As $t \to 1$, the curriculum schedule (controlled by $\texttt{exp\_base}=60$) triggers a sharp increase in penalty $\delta(t)$ and a collapse in reward variance $\sigma(t)$. This effectively tightens the failure boundaries, creating a steep potential barrier against invalid candidates. This regime minimizes the Generalization Risk by compelling the population to converge into the scientifically consistent subspace $\mathcal{H}_{\mathcal{C}}$.
\end{itemize}

This dynamic annealing mechanism does not impose strict constraints at the early stage of the search. Instead, it allows the model to explore the hypothesis space freely and gradually strengthens the prior constraints in later stages. This design effectively avoids the pseudo-equation trap, where models achieve good numerical fit but violate scientific principles, thereby ensuring that the final discovered equations strike a principled balance between numerical accuracy and scientific consistency.

\begin{figure*}[t]
    \centering
    \includegraphics[width=0.85\textwidth]{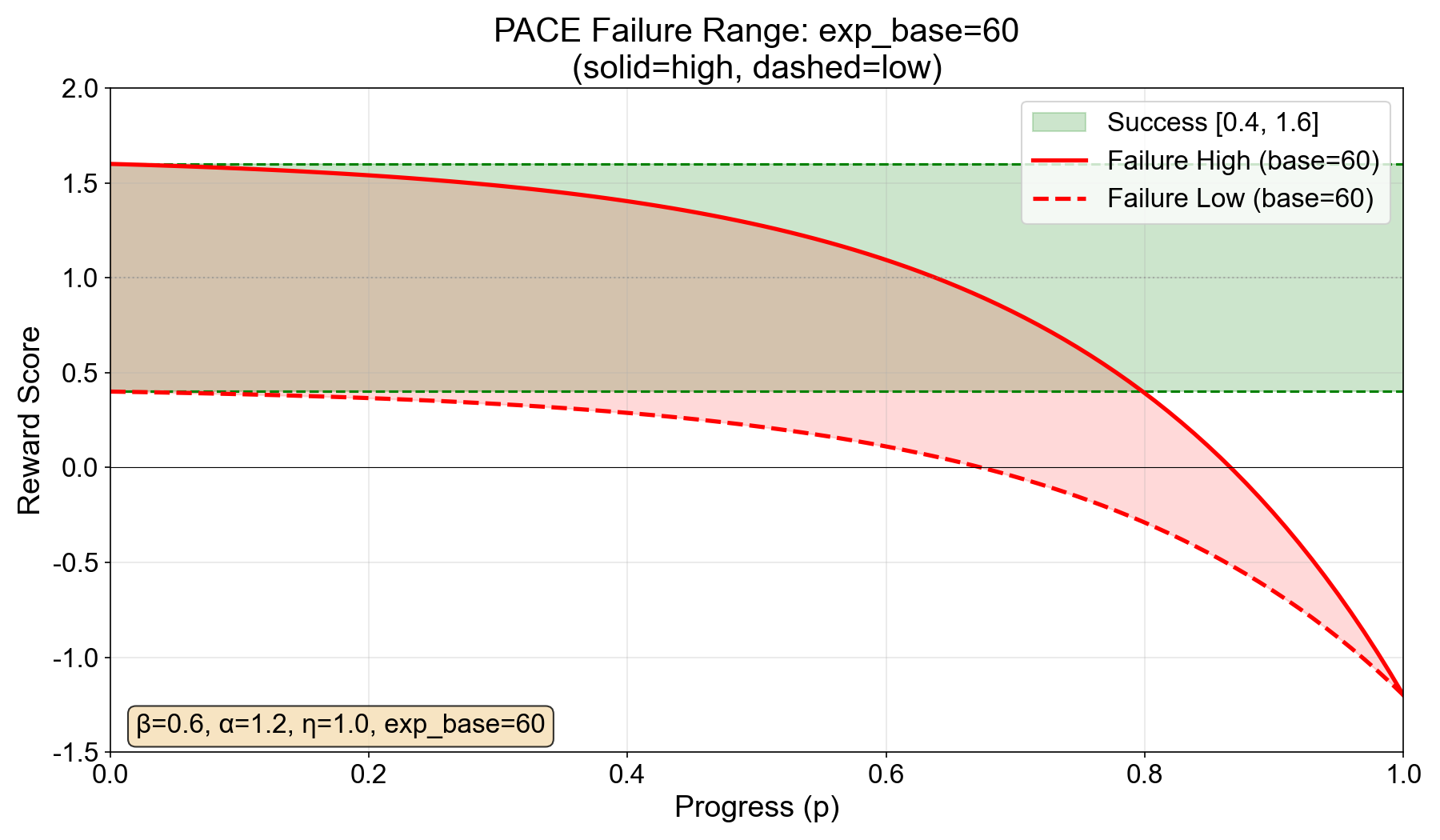}
    \caption{Dynamic score boundaries induced by the PACE (Prior-Annealing Constraint Evaluation) mechanism.}
    \label{FIG:8}
\end{figure*}

\newpage
\section{Details on Datasets}\label{app:dataset}

To simulate the real-world discovery process and prevent the risk of LLM recitation, our benchmarks are drawn from LLMSR-Bench~\cite{shojaee2025llm}.

\begin{itemize}
        \item \textbf{E. coli Growth}: 
    This dataset describes the growth dynamics of \textit{Escherichia coli} as a function of substrate concentration ($S$), temperature ($T$), and pH. The model adopts a multiplicative structure,
    \[
    \frac{dB}{dt} = f_B(B)\cdot f_S(S)\cdot f_T(T)\cdot f_{pH}(\mathrm{pH}),
    \]
    which captures the combined influence of biological capacity and environmental factors.  To avoid trivial recapitulation of standard biological growth models, we introduce novel nonlinear formulations for environmental dependencies. The explicit growth equation is given by
    \[
    \frac{dB}{dt} = \mu_{\max} B \left(\frac{S}{K_s + S}\right) \frac{\tanh\!\big(k(T - x_0)\big)}{1 + c(T - x_{\mathrm{decay}})^4} \exp\!\left(-\lvert \mathrm{pH} - \mathrm{pH}_{\mathrm{opt}} \rvert\right) \sin^2\!\left(\frac{(\mathrm{pH} - \mathrm{pH}_{\min})\pi}{\mathrm{pH}_{\max} - \mathrm{pH}_{\min}}\right).
    \]

    \item \textbf{Material Stress-Strain (Real-World Experimental Data)}: Unlike synthetic benchmarks, this dataset comprises actual experimental data from tensile tests of Aluminum 6061-T651 across a range of temperatures from $20^\circ\text{C}$ to $300^\circ\text{C}$. This problem challenges the model to recover empirical relationships from noisy, real-world observations where no universally accepted theoretical closed-form expression exists. The curves exhibit complex non-linearity, including distinct elastic, plastic, and failure regions that vary significantly with temperature. To rigorously test the out-of-domain generalization capabilities on real-world data, we allocate the data corresponding to $T=200^\circ\text{C}$ specifically for use as the out-of-domain validation set, withholding it entirely from the training process.

    \item \textbf{Chemical Reaction Kinetics (CRK)}: This benchmark describes how the concentration of a chemical species evolves over time. It incorporates standard exponential decay behavior alongside additional nonlinear saturation effects to evaluate model robustness in capturing governing equations for chemical decay and reaction rates.
    \begin{equation}
        \frac{dA}{dt} = -0.1899 \cdot A(t)^2 + \frac{0.4598\cdot A(t)^2}{0.7498 \cdot A(t)^4 + 1}
    \end{equation}
    
    where $A(t)$ denotes concentration at time $t$. This equation incorporates a second-order decay term alongside a synthetic nonlinear saturation term, evaluating symbolic reasoning in data-driven kinetic modeling.
    \item \textbf{Nonlinear Oscillators}: These systems follow the general differential form $\dot{x} + f(t, x, \dot{x}) = 0$, describing the complex interplay between an oscillator's position, velocity, and forces. Both datasets share the same time range $(0, 50)$ and initial values $\{x=0.5, v=0.5\}$. To effectively evaluate the generalization capability of the discovered equations, we employ a strategic data partitioning scheme where the simulation data is divided into training, in-domain validation, and out-of-domain validation sets based on the trajectory time. Specifically, the time interval $T=[0, 20)$ is utilized to evaluate out-of-domain generalization.
    \begin{itemize}
        \item \textbf{Oscillator 1}: A time-independent system governed by the specific equation:
        \[ \dot{v} = 0.8 \sin(x) - 0.5 v^3 - 0.2 x^3 - 0.5 x v - x \cos(x) \]
        It combines trigonometric, polynomial, and mixed terms to test the recovery of complex nonlinear interactions without temporal forcing.
        \item \textbf{Oscillator 2}: A time-dependent system governed by the specific equation:
        \[ \dot{v} = 0.3 \sin(t) - 0.5 v^3 - x v - 5.0 x \exp(0.5 x) \]
        This introduces explicit temporal forcing and exponential nonlinearities to challenge the model's ability to handle time-varying dynamics.
    \end{itemize}
\end{itemize}
\section{Summary of Prior Constraints}
\label{app:constraints_table}
\label{app:priors}

Table~\ref{tab:prior_constraints} lists the explicit prior constraints used in our method. As shown in Figure~\ref{fig:prior_constraints_summary}, these constraints are constructed using an LLM-assisted workflow: general scientific principles suggested by LLMs are checked and adjusted based on analyses of the training data. This process converts domain knowledge into executable constraint programs. For example:
\begin{itemize}[leftmargin=*, topsep=2pt, itemsep=3pt]

    \item \textbf{E. coli Growth:} Fundamental biological constraints, including near-zero growth at vanishing population density and near-zero growth under lethal environmental conditions (e.g., extreme temperature or pH), are observed in the training data. Additionally, unimodal responses to environmental factors coincide with the density truncation seen in the dataset.

    \item \textbf{Stress--Strain:} Due to noise, constraints are not imposed pointwise but derived from statistical trends in the data combined with established mechanical knowledge of Al 6061-T651. Specifically, these constraints include monotonic stress growth in the elastic and early plastic regimes, retention of load-carrying capacity at large deformation, observable thermal softening at elevated temperatures, and a bounded near-zero stress level at small strain.
    
    \item \textbf{CRK:} Reaction dynamics, including vanishing reaction rates at the equilibrium state, are consistently observed in the data. Moreover, the equilibrium point itself is identifiable from the concentration trajectories, which monotonically converge toward a steady state, providing strong empirical support for equilibrium consistency and global stability.

    \item \textbf{Oscillator-1:} Restoring forces toward equilibrium and velocity-opposing damping follow classical physical intuition and are supported by symmetric trajectories and monotonic amplitude decay. Phase portraits and joint state statistics further reveal clear departures from linear dynamics, while bounded trajectories are consistent with the finite range of the observed state distribution.

    \item \textbf{Oscillator-2:} Asymmetric restoring behavior is evident from the response statistics, which distinguish between positive and negative displacements. Furthermore, stable periodic structures in both time series and phase portraits confirm explicit periodic driving. Nonlinearity is substantiated by curvature and state-dependent distortions in phase space, alongside higher-order statistical dependencies in the data.
\end{itemize}

\begin{table}[htbp]
\centering
\caption{Summary of prior constraints for each  problem.}
\label{tab:prior_constraints}
\small
\renewcommand{\arraystretch}{1.4} 

\begin{tabular}{@{}>{\centering\arraybackslash}m{3.2cm} m{11.5cm}@{}}
\toprule
\textbf{Problem} & \textbf{Prior Constraints} \\
\midrule

\raisebox{1.2em}{%
    \begin{minipage}{3.2cm}
    \centering
    \textbf{E. coli Growth} 
    \end{minipage}%
} &
\begin{itemize}[leftmargin=*, nosep]
    \item Multivariate dynamics: $db = f(b, s, T, \text{pH})$
    \item Biological causality: $db = 0$ when $b = 0$
    \item Viability boundaries: $db \le 0$ at lethal conditions (eg. extreme Temperature or pH)
    \item Unimodal response to temperature and pH
    \item Asymmetric temperature response (sharper decay at high $T$)
\end{itemize} \\

\midrule

\raisebox{1.2em}{%
    \begin{minipage}{3.2cm}
    \centering
    \textbf{Stress--Strain}
    \end{minipage}%
} &
\begin{itemize}[leftmargin=*, nosep]
    \item Thermo-mechanical coupling: $\sigma = f(\varepsilon, T)$
    \item Near-zero stress behavior in the small-strain regime
    \item Monotonic stress growth across elastic and early plastic regimes with work hardening
    \item Statistically observable thermal softening at elevated temperatures
    \item Bounded and numerically stable model outputs (no NaN/Inf)
\end{itemize} \\

\midrule

\raisebox{1.2em}{%
    \begin{minipage}{3.2cm}
    \centering
    \textbf{CRK} 
    \end{minipage}%
} &
\begin{itemize}[leftmargin=*, nosep]
    \item Reaction dynamics system: $dA/dt = f(A,t)$
    \item Equilibrium consistency at $A_{\mathrm{eq}}$
    \item Global stability toward equilibrium
    \item Non-negativity constraint: $dA/dt >= 0$ at $A = 0$
    \item Mandatory nonlinearity
\end{itemize} \\
\midrule
\raisebox{1.2em}{%
    \begin{minipage}{3.2cm}
    \centering
    \textbf{Oscillator-1} 
    \end{minipage}%
} &
\begin{itemize}[leftmargin=*, nosep]
    \item State-dependent dynamics: $a = f(x, v)$, no explicit time dependence
    \item Restoring mechanism toward equilibrium 
    \item Dissipative damping that opposes velocity
    \item Nonlinearity requirement
    \item Bounded trajectories under typical initial conditions
\end{itemize} \\

\midrule

\raisebox{1.2em}{%
    \begin{minipage}{3.2cm}
    \centering
    \textbf{Oscillator-2} 
    \end{minipage}%
} &
\begin{itemize}[leftmargin=*, nosep]
    \item Non-autonomous dynamics: $a = f(t, x, v)$
    \item Restoring mechanism toward equilibrium and asymmetric restoring force
    \item Dissipative damping opposing velocity
    \item Bounded time-dependent driving term
    \item Nonlinearity requirement
    \item Bounded trajectories under typical initial conditions
\end{itemize} \\

\bottomrule
\end{tabular}
\end{table}

\begin{figure}
    \centering
    \includegraphics[width=0.8\textwidth]{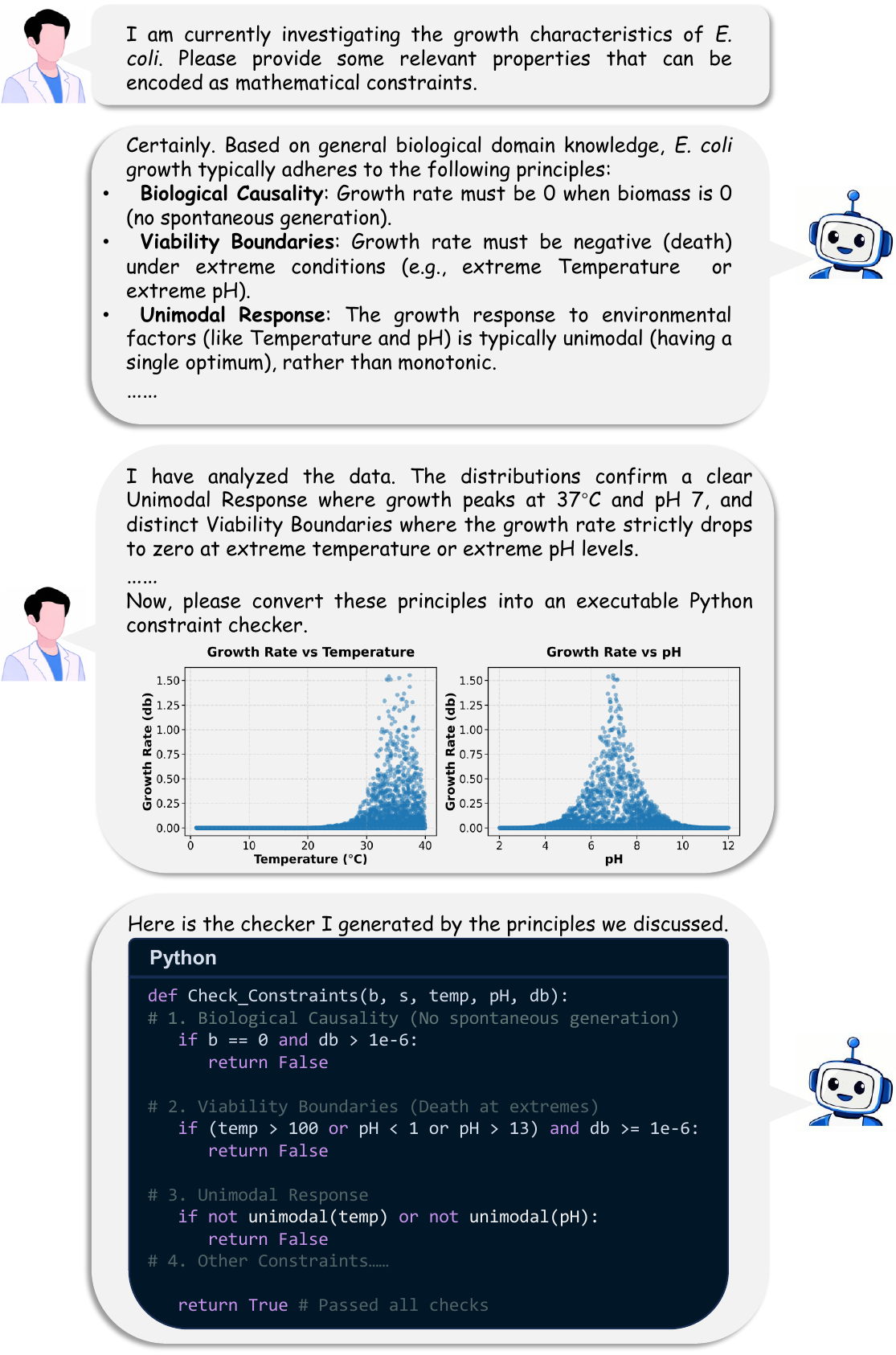}
\caption{Illustration of the prior constraint checker in PG-SR, using the \textit{E. coli Growth} task as an example to encode domain priors as executable constraint programs.}
    \label{fig:prior_constraints_summary}
\end{figure}

\section{Details on Baselines}
\label{app:baselines}

We compare PG-SR against three distinct categories of state-of-the-art symbolic regression approaches. The underlying principles and specific implementation settings for each baseline are detailed as follows:

\subsection{Search-based approaches}
These approaches formulate symbolic regression as a combinatorial optimization problem, searching for the optimal expression structure through evolutionary algorithms or reinforcement learning.

\begin{itemize}
    \item \textbf{GPLearn\cite{mccormick2019gplearn}}: A standard implementation of Genetic Programming (GP) that evolves a population of expression trees using genetic operators such as crossover, mutation, and reproduction. 
    \begin{itemize}
        \item \textit{Settings}: Population size is set to 1000, with 2000 generations and a tournament selection size of 20.
        \item \textit{Function set}: The function set consists of 
        \{\texttt{add}, \texttt{sub}, \texttt{mul}, \texttt{div}, 
        \texttt{sqrt}, \texttt{log}, \texttt{exp}, \texttt{sin}, \texttt{cos}, \texttt{abs}, \texttt{neg}, \texttt{inv}\}.
        \item \textit{Other settings}: All other hyperparameters follow the default settings of the original implementation.
    \end{itemize}

    \item \textbf{PySR~\cite{cranmer2023interpretable}}: An advanced SR method that employs asynchronous multi-island GP-based evolutions.
    \begin{itemize}
        \item \textit{Settings}: Population size is set to 50, with 2000 generations and a tournament selection size of 10. 
        \item \textit{Function set}: The function set consists of 
        \{\texttt{add}, \texttt{sub}, \texttt{mul}, \texttt{div}, 
        \texttt{sqrt}, \texttt{log}, \texttt{exp}, \texttt{sin}, \texttt{cos}, \texttt{abs}, \texttt{neg}, \texttt{inv}\}.
        \item \textit{Other settings}: All other hyperparameters follow the default settings of the original implementation.
    \end{itemize}
    
    \item \textbf{DSR (Deep Symbolic Regression)~\cite{petersen2019deep}}: Employs an RNN-based policy to sample mathematical expressions in the form of pre-order traversal sequences, optimized via a risk-seeking policy gradient.
    \begin{itemize}
        \item \textit{Settings}: The number of samples is set to 200{,}000, with a batch size of 100 and an exploration parameter $\epsilon = 0.05$.
        \item \textit{Function set}: The function set consists of 
        \{\texttt{add}, \texttt{sub}, \texttt{mul}, \texttt{div}, 
        \texttt{sqrt}, \texttt{log}, \texttt{exp}, \texttt{sin}, \texttt{cos}, \texttt{abs}, \texttt{neg}, \texttt{inv}\}.
        \item \textit{Other settings}: All other hyperparameters follow the default settings of the original implementation.
    \end{itemize}

    \item \textbf{uDSR~\cite{landajuela2022unified}}: A hybrid framework that unifies deep reinforcement learning with genetic programming, combining the global exploration capability of RL with the local search efficiency of GP.
    \begin{itemize}
        \item \textit{Settings}: The number of samples is set to 200{,}000, with a batch size of 100 and an exploration parameter $\epsilon = 0.05$.
        \item \textit{Prior}: A length prior is enabled (\texttt{on = true}), with the maximum expression length set to 20.
        \item \textit{Function set}: The function set consists of 
        \{\texttt{add}, \texttt{sub}, \texttt{mul}, \texttt{div}, 
        \texttt{sqrt}, \texttt{log}, \texttt{exp}, \texttt{sin}, \texttt{cos}, \texttt{abs}, \texttt{neg}, \texttt{inv}\}.
        \item \textit{Polynomial optimizer}: The polynomial optimizer parameters are set to a maximum degree of 2 and a coefficient tolerance of $1\times10^{-5}$.
        \item \textit{Other settings}: All other hyperparameters follow the default settings of the original implementation.
    \end{itemize}
\item \textbf{Operon~\cite{burlacu2020operon}}: A highly optimized C++ framework that uses a diverse set of selection and mutation schemes.
\begin{itemize}
    \item \textit{Settings}: All hyperparameters follow the default settings of the original implementation.
\end{itemize}

\item \textbf{GP-GOMEA~\cite{virgolin2021improving}}: A model-based GP approach that identifies linkage groups between variables to effectively preserve and combine useful sub-expressions.
\begin{itemize}
    \item \textit{Settings}: All hyperparameters follow the default settings of the original implementation.
\end{itemize}

\end{itemize}
\subsection{Transformer-based approaches}
These approaches enable end-to-end equation generation by leveraging large-scale pretraining on massive synthetic datasets.

\begin{itemize}
    \item \textbf{TPSR (Transformer-based Planning for Symbolic Regression)~\cite{shojaee2023transformer}}: Integrates Monte Carlo Tree Search into the decoding process of a pretrained Transformer, allowing the model to iteratively plan and refine the structure of symbolic expressions.
    \begin{itemize}
        \item \textit{Backbone}: A pretrained E2E Transformer model is used as the backbone network.
        \item \textit{Settings}: All other hyperparameters follow the default configuration of the original implementation.
    \end{itemize}

    \item \textbf{E2E (End-to-End Symbolic Regression)~\cite{kamienny2022end}}: A sequence-to-sequence pre-trained Transformer architecture for symbolic regression.
    \begin{itemize}
        \item \textit{Settings}: All other hyperparameters follow the default settings of the original implementation.
    \end{itemize}
    
    \item \textbf{NeSymReS~\cite{biggio2021neural}}: A pioneering pre-trained Transformer-based symbolic regression model.
    \begin{itemize}
        \item \textit{Model configuration}: Experiments are conducted using the 100M-parameter pretrained checkpoint.
        \item \textit{Other settings}:  All other hyperparameters follow the default settings of the original implementation.
    \end{itemize}
    
   \item \textbf{PhyE2E~\cite{ying2025neural}}: A recent sequence-to-sequence pre-trained Transformer model for symbolic regression.
    \begin{itemize}
        \item \textit{Settings}: All other hyperparameters follow the default settings of the original implementation.
    \end{itemize}
\end{itemize}

\subsection{LLM-based approaches}
These approaches leverage LLMs as program generators or evolutionary operators to guide the search for symbolic expressions.

\begin{itemize}
    \item \textbf{LaSR~\cite{grayeli2024symbolic}}: A Julia-based framework that integrates genetic algorithms with LLM.
    \begin{itemize}
        \item \textit{Settings}: All  hyperparameters follow the default configuration of the original implementation.
    \end{itemize}

    \item \textbf{LLM-SR~\cite{shojaee2024llm}}: An LLM-driven symbolic regression framework.
\begin{itemize}
    \item \textit{Settings}: All remaining hyperparameters follow the default configuration of the original implementation.
\end{itemize}

    \item \textbf{DrSR~\cite{wang2025drsr}}: An LLM-driven symbolic regression framework with dual reasoning mechanisms.
    \begin{itemize}
        \item \textit{Settings}: All  hyperparameters follow the default settings of the released codebase.
    \end{itemize}

\end{itemize}

\vspace{15cm}
\section{PG-SR Configuration}
\label{app:pg-config}
Table~\ref{tab:pgsr_full_config} presents the specific configurations.

\begin{table}[H]
\centering
\renewcommand{\arraystretch}{1.1}
\small
\caption{Complete Configuration Overview of PG-SR.}
\label{tab:pgsr_full_config}
\begin{tabular}{l c p{8cm}}
\toprule
\textbf{Parameter / Flag} & \textbf{Default} & \textbf{Description} \\
\midrule

\multicolumn{3}{c}{\textit{\textbf{Core Experiment Configuration}}} \\
\midrule
\texttt{num\_samplers} & 1 & Number of parallel samplers \\
\texttt{num\_evaluators} & 1 & Number of parallel evaluators \\
\texttt{samples\_per\_prompt} & 4 & Number of candidate hypotheses generated per prompt \\
\texttt{evaluate\_timeout\_seconds} & 30 & Timeout for evaluating a single hypothesis (seconds) \\
\texttt{api\_model} & gpt-4o-mini & LLM used for generation \\
\texttt{keep\_docstrings} & True & Whether to retain docstrings in extracted code \\
\texttt{max\_sample\_num} & 10,000 & Global maximum number of samples \\

\midrule
\multicolumn{3}{c}{\textit{\textbf{Experience Buffer Configuration}}} \\
\midrule
\texttt{functions\_per\_prompt} & 2 & Number of historical hypotheses included in each prompt \\
\texttt{num\_islands} & 10 & Number of islands (sub-populations) maintained in parallel \\
\texttt{reset\_period} & 4 hours & Period for resetting underperforming islands \\
\texttt{cluster\_sampling\_temp\_init} & 0.1 & Initial temperature for cluster-based sampling \\
\texttt{cluster\_sampling\_temp\_period} & 30,000 & Temperature decay period (in sampling steps) \\

\midrule
\multicolumn{3}{c}{\textit{\textbf{Prior Constraint \& PACE Configuration}}} \\
\midrule
\texttt{beta} & 0.6 & Success reward range $[1-\beta, 1+\beta]$ \\
\texttt{pace\_alpha} & 1.2 & Maximum downward shift for failure region \\
\texttt{pace\_eta} & 1.0 & Maximum shrinkage ratio for failure region \\
\texttt{pace\_exp\_base} & 60.0 & Exponential base controlling annealing speed \\
\texttt{n\_l-bfgs-b\_retries} & 10 & L-BFGS-B retries when physics check fails \\
\texttt{max\_repair\_rounds} & 3 & Maximum repair rounds per equation \\

\midrule
\multicolumn{3}{c}{\textit{\textbf{Warmup Stage Configuration}}} \\
\midrule
\texttt{warmup\_num\_skeletons} & 100 & Total number of skeletons requested during warmup \\
\texttt{warmup\_skeletons\_per\_call} & 5 & Number of skeletons per LLM call \\
\texttt{min\_physics\_passed\_warmup} & 10 & Minimum physics-passed skeletons required \\
\texttt{max\_warmup\_repair\_iterations} & 40 & Maximum repair iterations in warmup \\

\midrule
\multicolumn{3}{c}{\textit{\textbf{Refinement Stage Configuration}}} \\
\midrule
\texttt{refine\_start\_ratio} & 0.01 & Ratio of budget after which refinement starts \\
\texttt{refine\_interval} & 100 & Trigger refinement every N samples \\
\texttt{refine\_num\_skeletons} & 10 & Total number of skeletons requested per refinement \\
\texttt{refine\_skeletons\_per\_call} & 2 & Number of skeletons per LLM call \\
\texttt{max\_refine\_repair\_iterations} & 6 & Maximum repair iterations in refinement \\
\texttt{repair\_good\_examples\_num} & 3 & Max good examples shown in repair prompt \\
\texttt{repair\_history\_num} & 3 & Max repair history shown in repair prompt \\

\bottomrule
\end{tabular}
\end{table}

\newpage

\section{Details on Prompt Designs}
\label{app:prompts}

We provide the concrete prompt templates used in our experiments below.

\subsection{Data-Driven Warm-up Prompt}
\begin{promptbox}{Warmup Prompt}
\textbf{[Task: Data-Driven Scientific Hypothesis Generation]}

\textbf{Role}: You are a helpful assistant tasked with discovering mathematical function structures for \textbf{\{domain\}}. 

\textbf{Goal}: Based on the Analysis Report, Variable Definitions and the Hard Rules below, produce \textbf{EXACTLY \{num\_funcs\}} candidate \{equation\_desc\} (\{equation\_form\}).

\vspace{0.3em}
\textbf{[Variable Definitions]}

\{variable\_definitions\}

\vspace{0.3em}
\textbf{[Analysis Report]}

\{analysis\_report\}

\vspace{0.3em}
\textbf{[Hard Rules]}

\{hard\_rules\_block\}

\vspace{0.3em}
You must reason step-by-step. 
For each equation, provide a brief physical justification in the code comments (\# Reasoning: ...) BEFORE writing the formula.

\vspace{0.3em}
\textbf{[Output Format]}

\begin{lstlisting}[language=Python,basicstyle=\ttfamily]
import numpy as np
MAX_NPARAMS = 10
PARAMS_INIT = [1.0]*MAX_NPARAMS

def equation_v1({signature_vars}, params):
    '''
    Args:
        {signature_desc}
        params: Array of parameters [params[0], params[1], ...] 
                to be optimized.
    Returns:
        {return_var}: {return_desc}
    '''
    # Reasoning: ... 
    {return_var} = {example_formula1}
    return {return_var}
\end{lstlisting}
\end{promptbox}

\newpage


\subsection{Evolution Prompt}
\begin{promptbox}{Evolution Prompt}
\textbf{[System Message]}

You are a helpful assistant tasked with discovering mathematical function structures for scientific systems.

Complete the 'equation' function below, considering the physical meaning and relationships of inputs.

All parameters are constrained to be NON-NEGATIVE ($\geq$0). Write minus signs explicitly in the equation structure if needed.

Explain your reasoning briefly before completing the function. Let's think step by step.

\vspace{0.3em}
\textbf{[User Prompt - From Specification Template]}

\begin{lstlisting}[language=Python,basicstyle=\ttfamily\small]
"""
{problem_description_from_specification}

[Learnings from Recent Attempt]
The following are ideas summarized based on past experiences:
{reflection}

Use the above ideas and following equations to guide your equation 
improvement. Feel free to explore entirely different structures 
if they lead to better performance.
"""

import numpy as np

MAX_NPARAMS = 10
PARAMS_INIT = [1.0]*MAX_NPARAMS

{specification_code}

def equation_v0({signature_vars}, params):
    # Lower score version
    {return_var} = {example_formula_0}
    return {return_var}

def equation_v1({signature_vars}, params):
    # Medium score version
    {return_var} = {example_formula_1}
    return {return_var}

def equation_v2({signature_vars}, params):
    '''Improved version of `equation_v1`.'''
    # (Implementation to be generated)
\end{lstlisting}
\end{promptbox}

\newpage
\subsection{Residual Analysis Prompt}
\begin{promptbox}{Residual Analysis Prompt}
\textbf{[Task: Analyze Residual Patterns]}

\textbf{Role}: You are an expert data analyst specializing in symbolic regression diagnostics for \textbf{\{domain\}}.

\textbf{Goal}: Analyze the residual statistics and identify key patterns that explain why the current equation fails.

\vspace{0.3em}
\textbf{[Variable Definitions]}

\{variable\_definitions\}

\vspace{0.3em}
\textbf{[Current Equation]}

\begin{lstlisting}[basicstyle=\ttfamily]
{code_str}
\end{lstlisting}

\vspace{0.3em}
\textbf{[Raw Residual Statistics]}

\{raw\_statistics\}

\vspace{0.3em}
\textbf{[Analysis Instructions]}

Based on the statistics above, provide a concise analysis that:
\begin{itemize}[leftmargin=*, itemsep=0pt, topsep=2pt]
\item Identifies the PRIMARY deficiency of the current equation (what mathematical structure is missing or incorrect?)
\item Explains which input regions or variable ranges are most problematic
\item Analyze and summarize how changes of each independent variable influence the dependent variable, and the possible intrinsic relationships among independent variables.
\item Suggests what type of mathematical modifications would address these issues (e.g., nonlinear terms, interactions, saturation effects)
\end{itemize}

Be specific and actionable. Focus on structural insights rather than parameter tuning.

\vspace{0.3em}
\textbf{[Output Format]}

Provide your analysis briefly in the following JSON format:

\begin{lstlisting}[basicstyle=\ttfamily\small]
{
  "primary_deficiency": "Description of the main structural 
                         problem with the current equation",
  "problematic_regions": "Description of which input regions 
                          have the highest errors and why",
  "variable_relationships": "Analysis of how each independent 
                             variable ({input_vars}) influences 
                             the dependent variable ({output_var}) 
                             and their intrinsic relationships. 
                             Hint: analyze the functional relationship 
                             between each input variable and the output 
                             in different intervals",
  "suggested_modifications": "Specific mathematical structures or 
                              terms that should be added/modified"
}
\end{lstlisting}
\end{promptbox}

\newpage
\subsection{Refinement Prompt}

\begin{promptbox}{Refine User Prompt}
\textbf{[Task: Residual-Guided Equation Refinement]}

\textbf{Role}: You are a helpful assistant tasked with refining mathematical function structures for \textbf{\{domain\}}.

\textbf{Goal}: Based on the Residual Diagnostic Report, produce \textbf{EXACTLY \{refine\_n\}} refined versions of the current best equation.

\vspace{0.3em}
\textbf{[Variable Definitions]}

\{variable\_definitions\}

\vspace{0.3em}
\textbf{[Current Best Equation]}

\begin{lstlisting}[basicstyle=\ttfamily]
{code_str}
\end{lstlisting}

\vspace{0.3em}
\textbf{[Experience Hints]}

\{experience\_hints\}

(Optional: Retrieved from experience buffer based on similar equations.)

\vspace{0.3em}
\textbf{[Residual Diagnostic Report]}

\{analysis\_report\}

\vspace{0.3em}
\textbf{[Refinement Instructions]}

Based on the analysis above:
\begin{itemize}[leftmargin=*, itemsep=0pt, topsep=2pt]
\item Focus on addressing the PRIMARY DEFICIENCY identified in the report
\item Target the PROBLEMATIC REGIONS where the current equation fails most
\item Consider the VARIABLE RELATIONSHIPS to understand how each input affects the output in different regimes
\item Implement the SUGGESTED MODIFICATIONS to fix structural issues
\item Pay attention to \textbf{\{defect\_var\}}, which contributes most to prediction error
\item If the analysis suggests fundamental structural limitations, do not hesitate to explore completely different formulations
\end{itemize}

\vspace{0.3em}
\textbf{[Hard Rules]}

\{hard\_rules\_block\}

\vspace{0.3em}
\textbf{[Output Format]}

\begin{lstlisting}[basicstyle=\ttfamily]
import numpy as np

def equation_v1({signature_vars}, params):
    # Reasoning: ...
    {return_var} = ...
    return {return_var}
\end{lstlisting}
\end{promptbox}

\newpage
\subsection{Improvement Analysis Prompt}
\begin{promptbox}{Improvement Analysis Prompt}

\textbf{[Task: Analyze Equation Improvement]}

\textbf{Role}: You are an expert scientist analyzing mathematical equation improvements in \textbf{\{domain\}}.

\textbf{Goal}: Understand why a refined equation outperformed the original and extract actionable insights.

\vspace{0.3em}
\textbf{[Score Explanation]}
\{score\_explanation\}

\vspace{0.3em}
\textbf{[Variable Definitions]}
\{variable\_definitions\}

\vspace{0.3em}
\textbf{[Original Equation]} (score: \{original\_score\})
\begin{lstlisting}[basicstyle=\ttfamily]
{original_equation}
\end{lstlisting}

\vspace{0.3em}
\textbf{[Analysis of Residual Errors in the Original Equation]}
\{residual\_summary\}

(Note: These are the residual patterns/errors of the Original Equation that need fixing.)

\vspace{0.3em}
\textbf{[Improved Equation]} (score: \{improved\_score\})
\begin{lstlisting}[basicstyle=\ttfamily]
{improved_equation}
\end{lstlisting}

\vspace{0.3em}
\textbf{[Score Improvement]}: \{score\_gain\} (\{score\_gain\_percent\}\%)

\vspace{0.3em}
\textbf{[Analysis Instructions]}
\begin{enumerate}[leftmargin=*, itemsep=0pt, topsep=2pt]
\item What structural changes were made? (new terms, different function forms, interactions, etc.)
\item Explain how the structural changes corrected the specific patterns observed in the original equation's residuals.
\item Provide an actionable insight that could help improve similar equations.
\item Start your response with a \texttt{<thinking>} block to analyze the mathematical structural changes step-by-step.
\end{enumerate}

\vspace{0.3em}
\textbf{[Output Format]}

Your response must strictly follow this structure:

\textbf{PART 1: Thinking Process}

Wrap your step-by-step analysis inside a \texttt{<thinking>} tag.

Example: 
\begin{lstlisting}[basicstyle=\ttfamily]
<thinking> 
The original equation lacked... 
The new term interaction(x, y) captures... 
</thinking>
\end{lstlisting}

\textbf{PART 2: Structured JSON}

Provide the final result in a standard JSON block:
\begin{lstlisting}[basicstyle=\ttfamily]
{
  "insight": "Briefly describe the mathematical changes, 
  explain why they are effective in the domain context, 
  and conclude with one actionable takeaway for future 
  optimization."
}
\end{lstlisting}

\end{promptbox}

\newpage
\subsection{Island Reflection Prompt}
\begin{promptbox}{Island Reflection Prompt}
\textbf{[Task: Reflection Analysis]}

\textbf{Role}: You are an expert in scientific equation discovery and mathematical modeling.

\textbf{Domain}: \{domain\}

\vspace{0.3em}
\textbf{[Variable Definitions]}
\{variable\_definitions\}

\vspace{0.3em}
You are analyzing \{total\_equations\} equation candidates generated during a refinement process for the \{problem\_name\} problem.

\vspace{0.3em}
\textbf{[Original Equation (to be refined)]}
Score: \{baseline\_score\} (higher is better, score = -MSE)
\begin{lstlisting}[basicstyle=\ttfamily]
{original_equation}
\end{lstlisting}

\vspace{0.3em}
Below are the refined candidates. Each is labeled with a status showing whether it improved the score and whether it passed constraint checks:
\begin{itemize}[leftmargin=*, itemsep=0pt, topsep=2pt]
\item Improved/NotImproved: compared to baseline score
\item Valid/Invalid: constraint check results
\end{itemize}

\vspace{0.3em}
\textbf{[Refined Candidates Results]}
\begin{lstlisting}[basicstyle=\ttfamily]
{detailed_results}
\end{lstlisting}

\vspace{0.3em}
\textbf{[Task]}: Analyze these results and provide structured guidance for future evolution.
Compare the refined candidates against the original equation above.
You need to make your answer as concise as possible.

\vspace{0.3em}
Please provide your analysis in the following format:

\vspace{0.3em}
\textbf{What Worked Well}

[List 1-3 structural patterns or features that led to success. Please summarize useful experience.]
\begin{itemize}[leftmargin=*, itemsep=0pt, topsep=2pt]
\item Pattern 1: [describe what worked and why]
\item Pattern 2: [describe what worked and why]
\item Pattern 3: [describe what worked and why]
\end{itemize}

\vspace{0.3em}
\textbf{What Didn't Work}

[Analyze failures in two categories:]

\textit{Constraint Violations (Invalid)}

[List patterns that violated domain constraints. Please summarize the lessons you can draw from it.]
\begin{itemize}[leftmargin=*, itemsep=0pt, topsep=2pt]
\item Pattern 1: [describe what structural issue caused constraint violation and why]
\end{itemize}

\textit{Score Not Improved (NotImproved but Valid)}

[List patterns that passed constraints but didn't improve score. Please summarize what lessons can you draw from it.]
\begin{itemize}[leftmargin=*, itemsep=0pt, topsep=2pt]
\item Pattern 1: [describe what structural limitation prevented score improvement and why]
\item Pattern 2: [describe what structural limitation prevented score improvement and why]
\end{itemize}

\vspace{0.3em}
\textbf{Recommendations for Future Evolution}

[Provide 1-3 specific, actionable recommendations]
\begin{itemize}[leftmargin=*, itemsep=0pt, topsep=2pt]
\item Recommendation 1: [what to try or avoid]
\item Recommendation 2: [what to try or avoid]
\item Recommendation 3: [what to try or avoid]
\end{itemize}

Focus on general structural patterns, not specific parameter values. Be specific and actionable.

\end{promptbox}

\newpage
\subsection{Repair Prompt}

\begin{promptbox}{Repair User Prompt}
\textbf{[Task: Repair Refined Equation Based on Prior Constraint Violation]}

\textbf{Role}: You are a repair-oriented assistant tasked with fixing a failed mathematical equation structure that violates known constraints in the \textbf{\{domain\}}.

\textbf{Goal}: Using the Prior Constraints as the primary guide and the Residual Diagnostic Report as supporting evidence, produce \textbf{EXACTLY \{refine\_n\}} repaired variant(s) of the current best equation by modifying its structure to directly fix the identified failure.

\vspace{0.3em}
\textbf{[Variable Definitions]}
\{variable\_definitions\}

\vspace{0.3em}
\textbf{[Current Best Equation]}
\begin{lstlisting}[basicstyle=\ttfamily]
{code_str}
\end{lstlisting}

\vspace{0.3em}
\textbf{[Residual Diagnostic Report]}
\{analysis\_report\}

\vspace{0.3em}
\textbf{[Prior Constraints]}
\{physics\_block\}

Use the Prior Constraints as the primary reference to understand why the refined equation failed.
Use the Residual Diagnostic Report only as supporting evidence to identify which terms or variable dependencies may have caused the violation.
Focus your analysis on residual patterns most directly related to the failure reason, especially those involving \{defect\_var\}.

\vspace{0.3em}
\textbf{[Good Examples]}
\{good\_examples\_block\}

\vspace{0.3em}
\textbf{[Original Best Equation]}
\begin{lstlisting}[basicstyle=\ttfamily]
{original_equation}
\end{lstlisting}

\vspace{0.3em}
\textbf{[Failed Refined Equation]}
\begin{lstlisting}[basicstyle=\ttfamily]
{failed_equation}
\end{lstlisting}

\vspace{0.3em}
\textbf{[Failure Reason]}
\{failure\_reason\}

\vspace{0.3em}
\textbf{[Suggested Fix]}
\{fix\_hint\_block\}

\vspace{0.3em}
\textbf{[Previous Failed Attempts]}
\{repair\_history\_block\}

\vspace{0.3em}
\textbf{[Repair Instructions]}
\begin{enumerate}[leftmargin=*, itemsep=0pt, topsep=2pt]
\item Analyze WHY the equation failed based on the failure reason above and the prior constraints.
\item \{history\_instruction\}
\item Use no more than 10 parameters.
\end{enumerate}

\vspace{0.3em}
\textbf{[Output Format]}
\begin{lstlisting}[basicstyle=\ttfamily]
def equation({signature_vars}, params):
    # Reasoning: [explain how you fixed the prior violation]
    {return_var} = [your repaired formula]
    return {return_var}
\end{lstlisting}

\end{promptbox}

\newpage

\newpage
\section{Equations Discovered by PG-SR}
\label{app:discovered_equations}

\subsection{E. coli Growth: Equations Discovered by Our Method}

For the E. coli Growth dataset, the ground-truth equation discovered by PG-SR is:

\begin{equation}
\frac{dB}{dt} = \mu_{\max} B \left(\frac{S}{K_s + S}\right) \frac{\tanh\!\big(k(T - x_0)\big)}{1 + c(T - x_{\mathrm{decay}})^4} \exp\!\left(-\lvert \mathrm{pH} - \mathrm{pH}_{\mathrm{opt}} \rvert\right) \sin^2\!\left(\frac{(\mathrm{pH} - \mathrm{pH}_{\min})\pi}{\mathrm{pH}_{\max} - \mathrm{pH}_{\min}}\right),
\end{equation}

where $B$ is the bacterial population, $S$ the substrate concentration, $T$ the temperature, and $pH$ the acidity level.

Using our method, PG-SR initially discovered:
\begin{equation}
\frac{dB}{dt}
=
\theta\cdot
B \,
\frac{S}{S + K}
\cdot
\exp\!\left(-\frac{(37 - T)^2}{2\sigma_1^2}\right)
\cdot
\exp\!\left(-\frac{T - 37}{\sigma_2}\right)
\cdot
\left(1 + \frac{\max(T-35,\,0)}{\eta}\right)
\cdot
\left(1 - \frac{|\mathrm{pH}-7|}{\sigma_p}\right)^{\alpha}
\cdot
\left(\frac{B^2}{(B + \kappa)^2}\right)
,
\end{equation}

with parameters optimized as:
\begin{equation}
\begin{aligned}
K &= 0.99, \qquad
\eta = 4.06, \qquad
\sigma_1 = 2.69, \qquad
\sigma_2 = 2.72, \\
\sigma_p &= 6.42, \qquad
\alpha = 6.50, \qquad
\kappa = 7.15 \times 10^{-3}, \qquad
\theta = 0.31 .
\end{aligned}
\end{equation}

Structurally, the equation discovered by PG-SR follows a multiplicative growth formulation, integrating biomass-dependent growth, substrate limitation, temperature modulation, and pH-dependent inhibition, which is consistent with the ground-truth equation.
The temperature effect is modeled through an asymmetric response around a reference temperature, with an additional temperature-dependent amplification term introduced at elevated temperatures, again in agreement with the formulation of the ground-truth model.
The pH-dependent factor introduces a nonlinear inhibitory effect centered around neutral conditions, which is likewise consistent with the ground-truth equation.
In addition, the biomass-related term incorporates a nonlinear saturation effect to modulate growth at higher population densities.
Overall, the discovered equation is highly consistent with the ground-truth model in terms of the underlying growth mechanisms and key functional components, while differing in the specific functional forms employed.

\subsection{Stressstrain: Equations Discovered by Our Method}

For the Stressstrain dataset, no ground-truth equation is available. PG-SR discovered:
\begin{equation}
\sigma(\varepsilon, T)
=
\mathbb{I}\!\left(\varepsilon \le \frac{\sigma_y(\varepsilon,T)}{E_0}\right)\, (E_0\,\varepsilon)
+
\mathbb{I}\!\left(\varepsilon > \frac{\sigma_y(\varepsilon,T)}{E_0}\right)
\left[
\sigma_y(\varepsilon,T)\Big(1 + h\,(\varepsilon - \tfrac{\sigma_y(\varepsilon,T)}{E_0})\Big)
-
\eta\,(\varepsilon - \tfrac{\sigma_y(\varepsilon,T)}{E_0})^2
\right],
\end{equation}

\begin{equation}
\sigma_y(\varepsilon,T)
=
\max\!\left(
\sigma_{y0}\Big(1 - (a_1 T + a_2 T^2 + a_3 T^3 + a_4\,\varepsilon T)\Big),
\ \sigma_{\mathrm{sat}}
\right).
\end{equation}

with parameters optimized as:
\begin{equation}
\begin{aligned}
E_0 &= 36.468489, \qquad
\sigma_{y0} = 0.833684, \qquad
a_1 = 0.000000, \qquad
a_2 = 0.415737, \\
a_3 &= 0.387977, \qquad
a_4 = 0.101623, \qquad
\sigma_{\mathrm{sat}} = 0.255751, \\
\eta &= 0.443807, \qquad
h = 0.397738.
\end{aligned}
\end{equation}
The discovered constitutive relation exhibits a clear elastic–plastic structure with temperature-dependent yielding. The elastic regime follows a linear Hookean response, while the plastic regime is governed by a competition between linear hardening and quadratic softening. The yield stress decreases nonlinearly with temperature and is bounded from below by a saturation stress, ensuring physical consistency at high temperatures. Notably, the model identifies a coupling term between strain and temperature in the yield function, indicating that thermal softening effects intensify at larger strains. Overall, the discovered equation aligns well with established physical understanding of thermomechanical material behavior and is broadly consistent with the known constitutive characteristics of Al 6061-T651, despite the absence of ground-truth governing equations.

\subsection{CRK: Equations Discovered by Our Method}

For the concentration dataset, the ground-truth governing equation is
\begin{equation}
\frac{dA}{dt}
= -0.1899\,A^2
+ \frac{0.4598\,A^2}{0.7498\,A^4 + 1},
\end{equation}
where $A$ denotes the substance concentration.

PG-SR initially discovered the following highly flexible expression:
\begin{equation}
\frac{dA}{dt}
=
A^{\frac{p_0}{
\big( (A^A)^A - p_1 \big)^{
\big( (A^{-p_2} + p_3)^{(A + p_4)} \big)
}}}
\left(
- p_5 \exp(A) + p_6
\right).
\end{equation}

The optimized parameters (all positive) are:
\[
p_0 = 0.9216, \quad
p_1 = 0.1439, \quad
p_2 = 0.8598, \quad
p_3 = 0.1417, \quad
p_4 = 0.0136, \quad
p_5 = 0.1417, \quad
p_6 = 0.4579.
\]

Although the equation skeleton discovered by PG-SR differs substantially from the ground-truth equation in its structure, it nevertheless represents a highly expressive functional form that not only fits the observed data effectively, but also reproduces the characteristic dynamics in the vicinity of equilibrium points, thereby satisfying the key steady-state properties and local dynamical consistency constraints observed in the system.

\subsection{Oscillator1: Equations Discovered by Our Method}

For the \texttt{oscillator1} dataset, the ground-truth governing equation of the damped nonlinear oscillator is
\begin{equation}
\dot{v}
= 0.8\,\sin(x)
- 0.5\,v^3
- 0.2\,x^3
- 0.5\,x v
- x\cos(x),
\end{equation}
where $x$ denotes the position and $v$ denotes the velocity.

PG-SR initially discovered the following candidate equation:
\begin{equation}
\begin{aligned}
\ddot{x} ={}&
- c_1 v
- c_2 \sin(x)
+ c_3 x^3
+ c_4 v
- c_5 v^2
- c_6 x v \\
&- c_7 x^5
- c_8 v^3
+ c_9
- c_{10} \sin(2x).
\end{aligned}
\end{equation}

The optimized parameters (all positive) are:
\begin{equation}
\begin{aligned}
c_1 &= 0.5503, \quad
c_2 = 0.0954, \quad
c_3 = 0.0810, \quad
c_4 = 0.5503, \\
c_5 &= 1.68 \times 10^{-7}, \quad
c_6 = 0.5000, \quad
c_7 = 0.0203, \\
c_8 &= 0.5000, \quad
c_9 = 5.49 \times 10^{-9}, \quad
c_{10} = 0.0523.
\end{aligned}
\end{equation}

Negligible terms were removed, including the canceling linear damping terms $-c_1 v + c_4 v$, as well as small-magnitude terms $c_5$, $c_7$, $c_9$, and $c_{10}$.
This yields the simplified PG-SR equation:
\begin{equation}
\dot{v} \approx
-0.0954\,\sin(x)
+ 0.0810\,x^3
- 0.5000\,x v
- 0.5000\,v^3.
\end{equation}

The simplified PG-SR equation successfully reproduces the dominant nonlinear oscillator dynamics by capturing cubic nonlinearities in both position and velocity, the interaction between position and velocity, and sinusoidal driving. Although the recovered coefficients differ slightly from the ground truth, particularly for the $\sin(x)$ and $x^3$ terms, the discovered equation correctly identifies all key functional structures present in the true system.
\newpage
\subsection{Oscillator2: Equations Discovered by Our Method}

For the \texttt{oscillator2} dataset, the ground-truth governing equation is given by
\begin{equation}
\dot{v}
= 0.3\,\sin(t)
- 0.5\,v^3
- x v
- 5.0\,x \exp(0.5x).
\end{equation}

PG-SR initially discovered the following candidate equation:
\begin{equation}
\begin{aligned}
\dot{v} ={}& A_1 \sin(t)
+ A_2 \cos(t)
- k_1 x e^{0.5x}
- b_1 v
- b_2 v^3 \\
&- \alpha x v
- \beta v^2 x
- \delta |x| v
- \gamma v x^2
- \gamma x^2 .
\end{aligned}
\end{equation}

The optimized parameters are:
\[
\begin{aligned}
& k_1 = 5.00001, \quad
b_1 = 0.0, \quad
b_2 = 0.500037, \quad
A_1 = 0.299999, \quad
A_2 = 0.0, \\
& \alpha = 1.000045, \quad
\beta = 0.0, \quad
\delta = 0.0, \quad
\gamma = 0.0 .
\end{aligned}
\]

After removing negligible terms ($b_1$, $A_2$, $\beta$, $\delta$, and $\gamma$), the simplified PG-SR equation becomes
\begin{equation}
\dot{v} \approx
0.3\,\sin(t)
- 0.5\,v^3
- x v
- 5.0\,x \exp(0.5x).
\end{equation}

A direct comparison with the ground-truth equation shows that the discovered equation exactly matches the true governing equation. All principal components, including the driving term $\sin(t)$, the restoring force $x e^{0.5x}$, the cubic damping term $v^3$, and the interaction term $xv$, are recovered with identical coefficients.
These results demonstrate that PG-SR fully achieves scientific discovery by identifying the underlying functional structure governing the system dynamics.

\section{Computational Resources and Inference Cost}
\label{app:compute_cost}

According to the experimental setup, we use gpt-4o-mini as the backbone model. Each PG-SR experiment is configured with a global budget of 10,000 samples. All experiments were conducted on a local workstation equipped with an Intel Core i9-14900HX CPU and an NVIDIA GeForce RTX 4090 GPU.
Under this configuration, a complete run on a standard benchmark (e.g., Oscillator1) takes approximately 10 hours. The total token usage per run is approximately 13.36 million input tokens and 4.49 million output tokens, totaling roughly 17.85 million tokens. Based on standard commercial API pricing, the estimated cost per run is approximately \$4.70 USD.
Although this approach incurs higher time and monetary costs compared to traditional search-based approaches, it remains highly accessible for academic research without requiring expensive high-performance computing clusters. More importantly, given the substantial gains in scientific consistency and OOD generalization (as shown in Table~\ref{tab:sr_baseline_comparison}), this cost is negligible.

\end{document}